\documentclass[12pt]{article}
\pdfoutput=1  
\usepackage{scicite}

\usepackage{times}

\usepackage{amsfonts}
\usepackage{amssymb}
\usepackage{graphicx}

\usepackage{booktabs}       
\usepackage{graphicx}
\usepackage{mathtools}
\usepackage[usenames, dvipsnames]{color}
\usepackage{amsthm}
\usepackage{url}
\usepackage{epigraph}
\usepackage{xspace}
\usepackage{wrapfig}
\usepackage{tabularx}
\usepackage{algcompatible}  
\usepackage{algorithm}
\usepackage{algpseudocode}
\usepackage[a4paper, total={6in, 10in}]{geometry}
\usepackage[labelfont=bf]{caption}
\usepackage{subcaption}
\usepackage{arydshln}
\usepackage{soul}  
\usepackage{titlesec} 
\usepackage{hyperref}
\makeatletter
\def\ttl@useclass#1#2{%
  \@ifstar
    {\ttl@labelfalse\@dblarg{#1{#2}}}
    {\ttl@labeltrue\@dblarg{#1{#2}}}}
\makeatother

\usepackage{blindtext}

\newif\ifcamera
\camerafalse

\newif\ifshowchanges
\showchangesfalse

\newif\ifscience
\sciencetrue

\ifshowchanges
\newcommand{\CHANGED}[1]{{\color{red}#1}}
\else
\newcommand{\CHANGED}[1]{{#1}}
\fi

\newcommand{\sectionref}[1]{{\nameref{#1}}}
\newcommand{\supplementaryref}[1]{supplementary text}

\newcommand{\bp}{\mathbf{p}}

\newcommand{\bv}{\mathbf{v}}

\newcommand{\btheta}{{\boldsymbol\theta}}
\newcommand{\cA}{\mathcal{A}}

\newcommand{\cH}{\mathcal{H}}
\newcommand{\cL}{\mathcal{L}}

\newcommand{\cO}{\mathcal{O}}
\newcommand{\cP}{\mathcal{P}}
\newcommand{\cS}{\mathcal{S}}
\newcommand{\cT}{\mathcal{T}}

\newcommand{\spub}{s_{\text{pub}}}
\ifscience
\newcommand{\defword}[1]{#1} 
\newcommand{\emphword}[1]{#1}
\newcommand{\ie}{i.e.~}
\newcommand{\eg}{e.g.~}
\else
\newcommand{\defword}[1]{\textbf{\boldmath{#1}}}
\newcommand{\emphword}[1]{\textit{#1}}
\newcommand{\ie}{{\it i.e.}~}
\newcommand{\eg}{{\it e.g.}~}
\fi

\newtheorem{theorem}{Theorem}

\newtheorem{lemma}{Lemma}

\newcommand{\pog}{\CHANGED{\textsc{SoG}}\xspace}
\newcommand{\poglong}{\CHANGED{Student of Games}\xspace}

\newcommand{\cfv}[2]{v^{#2}(#1)}
\newcommand{\apxcfv}[2]{\tilde{v}^{#2}({#1})}
\newcommand{\brcfv}[2]{BV^{#2}({#1})}

\newcommand{\pubstate}{{s_\textnormal{pub}}}

\newcommand{\infostate}[1]{{s_{#1}}}

\newcommand{\gtcfr}{GT-CFR}
\newcommand{\treeinterior}[1]{\mathcal{N}({#1})}
\newcommand{\treefrontier}[1]{\mathcal{F}({#1})}
\newcommand{\treeterminals}[1]{\mathcal{Z}({#1})}
\newcommand{\interiorintervals}[1]{\mathcal{T}_n({#1})}

\newcommand{\subgame}[1]{{sub-game}}

\newenvironment{sciabstract}{%
\begin{quote} \bf}
{\end{quote}}

\title{\CHANGED{Student of Games: A unified learning algorithm for both perfect and imperfect information games}}

\author{Martin Schmid${}^{1,2}$,
Matej Morav\v{c}\'{i}k${}^{1,2}$,
Neil Burch${}^{2,3,4}$,
Rudolf Kadlec${}^{1,2}$,\\
Josh Davidson${}^{2,3}$,
Kevin Waugh${}^{2,3}$,
Nolan Bard${}^{2,3}$,
Finbarr Timbers${}^{2,5}$,\\
Marc Lanctot${}^{2,6}$,
G. Zacharias Holland${}^{2,3}$,
Elnaz Davoodi${}^{2,6}$,\\
Alden Christianson${}^{2,7}$,
Michael Bowling${}^{2,4,7\ast}$\\
\\
\CHANGED{\normalsize{${}^{1}$EquiLibre Technologies, Prague, Czechia}, 
\normalsize{${}^{2}$Google DeepMind},} \\
\CHANGED{\normalsize{${}^{3}$Sony AI, New York, NY, USA},
\normalsize{${}^{4}$Amii, Edmonton, Canada},} \\
\CHANGED{\normalsize{${}^{5}$Midjourney, South San Francisco, CA, USA},} \\
\CHANGED{\normalsize{${}^{6}$Google DeepMind, Montreal, Canada},
\normalsize{${}^{7}$University of Alberta, Edmonton, Canada}} \\
\CHANGED{\normalsize{$^\ast$To whom correspondence should be addressed; E-mail:  {{\tt }mbowling@ualberta.ca}}}
}

\date{}

\begin{document} 




\maketitle 

\begin{quote}
\item \paragraph{Teaser:} \poglong combines search, learning, and game-theoretic reasoning to play chess, go, poker, and Scotland Yard.
\end{quote}



\begin{sciabstract}
Games have a long history as benchmarks for progress in artificial intelligence.
Approaches using search and learning produced strong performance across many perfect information games, and approaches using game-theoretic reasoning and learning demonstrated strong performance for specific imperfect information poker variants.
We introduce \poglong, a general-purpose algorithm that unifies previous approaches, combining guided search, self-play learning, and game-theoretic reasoning. 
\poglong achieves strong empirical performance in large perfect and imperfect information games --- an important step towards truly general algorithms for arbitrary environments.
We prove that \poglong is sound, converging to perfect play as available computation and approximation capacity increases.
\poglong reaches strong performance in chess and Go, beats the strongest openly available agent in heads-up no-limit Texas hold'em poker, and defeats the state-of-the-art agent in Scotland Yard, an imperfect information game that illustrates the value of guided search, learning, and game-theoretic reasoning.
\end{sciabstract}


\section*{Introduction}

In the 1950s, Arthur L. Samuel developed a checkers-playing program that employed what is now called 
minimax search (with alpha-beta pruning) and ``rote learning'' to improve its evaluation function via self-play~\cite{Samuel59}.
This investigation inspired many others, and ultimately Samuel co-founded the field of artificial 
intelligence~\cite{RussellNorvig2010} and popularized the term ``machine learning''. 
A few years ago, the world witnessed a computer program defeat a long-standing professional at the game
of Go~\cite{Silver16Go}. AlphaGo also combined learning and search.
Many similar achievements happened in between, such as the race for super-human chess leading to DeepBlue~\cite{DeepBlue} and 
TD-Gammon teaching itself to play master-level performance in Backgammon through self-play~\cite{TDGammon}, continuing the tradition of using games
as canonical markers of mainstream progress across the field.

Throughout the stream of successes, there is an important common element: the focus on a single game. Indeed, DeepBlue could
not play Go, and Samuel's program could not play chess. Likewise, AlphaGo could not play chess; however its successor
AlphaZero~\cite{Silver18AlphaZero} could, and did. AlphaZero demonstrated that a single algorithm could master three different
perfect information games \CHANGED{--- where the game's state is known to all players ---} using a simplification of AlphaGo's approach, and with minimal human knowledge. Despite this success, AlphaZero could not play poker,
and the extension to imperfect information games was unclear. 

Meanwhile, approaches taken to achieve super-human poker AI were \CHANGED{substantially} different. 
Strong poker play has relied on game-theoretic reasoning to ensure that
private information is concealed effectively. 
Initially, super-human poker agents were based primarily on computing approximate Nash equilibria
offline~\cite{Johanson16phdthesis}.
Search was then added and proved to be a crucial ingredient
to achieve super-human success in no-limit variants~\cite{Moravcik17DeepStack,Brown17Libratus,Brown19Pluribus}.
Training for other large games have also been inspired by game-theoretic reasoning and search,
such as Hanabi~\cite{Bard19Hanabi,lerer2020hanabi}, The Resistance~\cite{serrino2019finding}, Bridge~\cite{lockhart2020humanagent},
AlphaStar~\cite{Silver19AlphaStar}, and (no-press)
Diplomacy~\cite{anthony2020learning,gray2020humanlevel,bakhtin2021nopress}.
Here again, however, despite remarkable success: each advance 
was still on a single game, with some clear uses of
domain-specific knowledge and structure to reach  strong performance.

In this paper, we introduce \poglong (\pog), an algorithm that generalizes the class of games in which strong performance can be achieved using self-play learning, search, and game-theoretic reasoning.
\pog uses growing-tree counterfactual regret minimization (GT-CFR): an anytime local search that builds subgames non-uniformly, expanding the tree toward the most relevant future states while iteratively refining values and policies.
In addition, \pog employs sound self-play: a learning procedure that trains value-and-policy networks using both game outcomes and recursive sub-searches applied to situations that arose in previous searches.

\poglong achieves strong performance in multiple challenge domains with both perfect and imperfect information -- an important step towards truly general algorithms that can learn in arbitrary environments.
Applications of traditional search suffer well-known problems in imperfect information games~\cite[Section 5.6.2]{RussellNorvig2010}. Evaluation has remained focused on single domains (e.g. poker) despite recent progress toward sound search in imperfect information games~\cite{Moravcik17DeepStack,Brown17Safe,sustr2020sound}.
\poglong fills this gap, using a single algorithm with minimal domain-specific knowledge. Its search is sound~\cite{sustr2020sound} across these fundamentally different game types: it is guaranteed to find an approximate Nash equilibrium by re-solving subgames to remain consistent
during online play, and yields low exploitability in practice in small games where exploitability is computable.
\pog demonstrates strong performance across
four different games: two perfect information (chess and Go) and two imperfect information (poker and Scotland Yard).
Finally, unlike poker, Scotland Yard has \CHANGED{substantially} longer search horizons and game lengths, requiring long-term planning.

\subsection*{Background and Terminology}
\label{sec:background}

\poglong will be presented using the Factored-Observation Stochastic Games (FOSG) formalism.  For further details on the formalism, see~\cite{KovarikFOSG,Schmid21Thesis}.

A game between two players starts in a specific \defword{world state} $w^{init}$ and proceeds to the successor world states $w \in \mathcal{W}$ as a result of players choosing actions $a \in \cA$ until the game is over when a terminal state is reached.
A world state can be categorized as a decision node, a terminal node, or a chance node. At a decision node, player $\cP(w)$ acts. A terminal node marks the end of a game where no players act. A chance node is a special node representing a stochastic event, such as a die roll, with a fixed distribution.
At any world state $w$, $\cA(w) \subseteq \cA$ refers to those actions that are available, or legal, in world state $w$.
Sequences of actions taken along the course of the game are called \defword{histories} and denoted $h \in \cH$, with $h' \sqsubseteq h$ denoting a prefix history (subsequence). At terminal histories, $z \subset \cH$, each player $i$ receives a utility $u_i(z)$.

An \defword{information state} is a state with respect to one player's information. Specifically, $s_i \in \cS_i$ for player $i$ is a set of histories that are indistinguishable due to missing information. A simple example is a specific decision point in poker where player $i$ does not know the opponent's private cards; the histories in the information state are different only in the chance event outcomes that determine the opponent's private cards, since everything else is public knowledge. 
A player $i$ plays a \defword{policy} $\pi_i : \cS_i \rightarrow \Delta(\cA)$, where $\Delta(\cA)$ denotes the set of probability distributions over actions $\cA$.
The goal of each player is to find a policy that maximizes their own expected utility.

Every time a player takes an action, each player gets a \defword{private observation} $\cO_{\text{priv}(i)}(w, a, w')$ and a \defword{public observation} $\cO_{\text{pub}}(w, a, w')$ as a result of applying action $a$, changing the game's state from $w$ to $w'$. 
\CHANGED{In \defword{perfect information} games, the public observation contains complete information, i.e.,  $\cO_{\text{pub}}(w, a, w') = (w, a, w')$ making any private observations uninformative.  Furthermore, the transition function depends only on the active player's action, i.e., $\cT(w, a) = \cT(w, a_{\cP(w)})$.
In contrast, \defword{imperfect information} games have information asymmetry between players and some players will receive informative private observations.}
A \defword{public state} $\spub(h) \in \cS_{\text{pub}}$ is the sequence of public observations encountered along the history $h$. For example, a public state in Texas hold'em poker is represented by initial public information (stack sizes and antes), the betting history, and any publicly revealed board cards.
Let $\cS_i(\spub)$ be the set of possible information states for player $i$ given $\spub$: each information state $s_i \in \cS_i(\spub)$ is consistent with public observations in $\spub$ but has different sequences of private observations. Supplementary material shows a full example of a FOSG in Figure~\ref{fig:fosg-example} and an example public tree in Figure~\ref{fig:fosg-public-tree}. 

\CHANGED{Imperfect information games introduce additional complexity, as $\cS_i(\spub)$ can now contain multiple information states that the player's policy depends on. For example, in poker the information states would contain the private cards of player $i$.  Since past actions can leak otherwise private information, agents must reason about which information states players could be in to act soundly.}
A \defword{public belief state} is defined as $\beta = (\spub, r)$, where the \defword{range} (or beliefs) $r \in \Delta(\cS_1(\spub)) \times \Delta(\cS_2(\spub))$ is a pair of distributions over possible information states for both players representing the beliefs over information states in $\spub$. 
A basic depiction of the various components of a public belief state is depicted in Figure~\ref{fig:public-belief-state}.

Suppose players use a policy profile $\pi = (\pi_1, \pi_2)$. Denote the expected utility to player to player $i$ as $u_i(\pi_1, \pi_2)$ and $-i$ as the opponent of player $i$.
A \defword{best response} to a specific opponent policy $\pi_{-i}$ is any policy $\pi_i^b$ that achieves maximal utility against $\pi_{-i}$:
$\pi_i^b \in \{ \pi_i~|~u_i(\pi_i, \pi_{-i}) = \max_{\pi_i'} u_i(\pi_i', \pi_{-i}) \}$. A policy profile $\pi$ is a \defword{Nash equilibrium} if and only if $\pi_1$ is a best response to $\pi_2$ and $\pi_2$ is a best response to $\pi_1$. There are also approximate equilibria: $\pi$ is an \defword{$\epsilon$-Nash equilibrium} if and only if $u_i(\pi_i^b, \pi_{-i}) - u_i(\pi_i, \pi_{-i}) \le \epsilon$ for all players $i$.

In two-player zero-sum games, Nash equilibria are optimal because they maximize worst-case utility guarantees for both players.  This worst-case utility is unique and is called the game's \defword{value}.  \CHANGED{For such games, a 
standard metric to represent empirical convergence rate is a strategy's \defword{exploitability}: how much, on average, a player will lose against a best response relative to a Nash equilibrium. For a given policy profile in a two-player zero-sum game $\pi = (\pi_1, \pi_2)$, $\textsc{Exploitability}(\pi) = (\max_{\pi_1'}{u_1(\pi_1', \pi_2)} + \max_{\pi_2'}{u_2(\pi_1, \pi_2')})/2$.}
Also, equilibrium strategies in two-player zero-sum games are \defword{interchangeable}: if $\pi^A$ and $\pi^B$ are Nash equilibria, then $(\pi_1^A, \pi_2^B)$ and $(\pi_1^B, \pi_2^A)$ are also both equilibria.
These properties mean that a Nash equilibrium plays perfect defence: it
will not lose on expectation against any opponent, even one that is playing a best response to the Nash equilibrium. If the opponent makes mistakes, then the Nash equilibrium policy can win. Thus, it is reasonable for an agent to compute and play
a Nash equilibrium, or an approximation to one with low exploitability.

\CHANGED{Although the FOSG formalism generalizes beyond two-player zero-sum games, the theoretical guarantee of Nash equilibria outside of this setting is less meaningful and it is unclear how effective they would be (for example, in games with more than two players). In this work, we focus on the two-player zero-sum setting.} \CHANGED{To put \pog in context, we begin with a high-level description of several techniques that have been dominant in this setting, and then contrast \pog with existing work.}

\subsubsection*{Tree Search and Machine Learning}

The first major milestones in the field of game-playing AI were obtained by efficient search techniques inspired by the minimax theorem~\cite{Samuel59,DeepBlue}. In a two-player zero-sum game with perfect information, the approach uses depth-limited search starting from the current world state $w_t$, along with a heuristic evaluation function to estimate \CHANGED{the value of} states beyond the depth limit, $h(w_{t+d})$, and game-theoretic reasoning to back up values~\cite{KnuthMoore75}. Researchers developed \CHANGED{notable} search enhancements~\cite{Marsland81,Schaeffer1996NewAI} that greatly improved performance, leading to  IBM's super-human DeepBlue chess program~\cite{DeepBlue}. 

This classical approach was, however, unable to achieve super-human performance in Go, which has \CHANGED{substantially} larger branching factor and state space complexity than chess. Prompted by the challenge of Go~\cite{Gelly12Go}, researchers proposed Monte Carlo tree search (MCTS)~\cite{Kocsis06UCT,Coulom06MCTS}. Unlike minimax search, MCTS builds trees via simulations, starting with an empty tree rooted by $w_t$ and expanding the tree by adding the first state encountered in simulated trajectories \CHANGED{that is} not currently in the tree, and finally estimating values from rollouts to the end of the game. MCTS led to \CHANGED{substantially} stronger play in Go and other games~\cite{Browne12}, attaining 6 dan amateur level in Go. However, heuristics leveraging domain knowledge were still necessary to achieve these milestones.

In AlphaGo~\cite{Silver16Go}, value functions and policies are incorporated, learned initially from human expert data, and then improved via self-play.
A deep network approximates the value function and a prior policy helps guide the selection of actions during the tree search. The approach was the first to achieve super-human level play in Go~\cite{Silver16Go}. AlphaGo Zero removed the initial training from human data and Go-specific features~\cite{Silver2017AGZ}. AlphaZero reached state-of-the-art performance in chess and Shogi as well as Go, using minimal domain knowledge~\cite{Silver18AlphaZero}.

\poglong, like AlphaZero, combines search and learning from self-play, using minimal domain knowledge. Unlike MCTS, which is not sound for imperfect information games, \pog's search algorithm is based on counterfactual regret minimization and is sound for both perfect and imperfect information games.

\subsubsection*{Game-Theoretic Reasoning and Counterfactual Regret Minimization}

In imperfect information games, the choice of strategies that arise from hidden information can be crucial to determining each player's expected rewards. 
Simply playing too predictably can be problematic: in the classic example game of Rock, Paper, Scissors, the only thing a player does not know is the choice of the opponent's action, however this information fully determines their achievable reward. A player choosing to always play one action (e.g. rock) can be easily beaten by another playing the best response (e.g. paper). The Nash equilibrium plays each action with equal probability, which minimizes the benefit of any particular counter-strategy.
Similarly, in poker, knowing the opponent's cards or their strategy could yield \CHANGED{substantially} higher expected reward, and in Scotland Yard, players have a higher chance of catching the evader if their current location is known.
In these examples, players can exploit any knowledge of hidden information to play the counter-strategy resulting in higher reward.
Hence, to avoid being exploited, players must act in a way that does not reveal their own private information. We call this general behavior \defword{game-theoretic reasoning} because it emerges as the result of computing (approximate) minimax-optimal strategies. Game-theoretic reasoning has been paramount to the success of competitive poker AI over the last 20 years. 

One algorithm for computing approximate optimal strategies is counterfactual regret minimization (CFR)~\cite{08nips-cfr}. 
CFR is a self-play algorithm that produces policy iterates $\pi^t_i(s,\cdot)\in \Delta(\cA)$ for each player $i$ at each of their information states $s$ in a way that minimizes long-term average regret. As a result, the (appropriately weighted) average policy over all $T$ iterations $\bar{\pi}^T$ converges to an $\epsilon$-Nash equilibrium at a rate of $O(1/\sqrt{T})$. At each iteration, $t$, counterfactual values $v^t_i(s,a)$ are computed for each action $a \in \cA(s)$ and immediate regrets for not playing $a$, $r^t(s,a) = v^t_i(s,a) - \sum_{a \in \cA(s)}{ \pi^t(s,a) v^t_i(s,a)}$, are computed and tabulated in a cumulative regret table storing $R^T(s,a) = \sum_{t=1}^T{r^t}(s,a)$. A new policy is computed using regret-matching~\cite{Hart00}: $\pi^{t+1}(s,\cdot) = \frac{[R^t(s,\cdot)]^+}{\sum_a{[R^t(s,a)]^+}}$, where $[x]^+ = \max(x, 0)$, if $\sum_a{[R^t(s,a)]^+} > 0$, or the uniform distribution otherwise.

CFR$^+$~\cite{Tammelin15CFRPlus} is a successor of CFR that played a key role in solving the game of heads-up limit hold'em poker, the largest imperfect information game to be solved to date~\cite{Bowling15Poker}. A key change in CFR$^+$ is a different policy update mechanism,  regret-matching$^+$, which defines cumulative values slightly differently: $Q^t(s,a) = (Q^{t-1}(s,a) + r^t(s,a))^+$, and $\pi^{t+1}(s,a) = Q^t(s,a) / \sum_b{Q^t(s,b)}$.

A common form of CFR (or CFR$^+$) is one that traverses the \defword{public tree} of public states, rather than the classical extensive-form game tree of world states (and information states).
Quantities required to compute counterfactual values, such as each player's probabilities of reaching each information state under their policy (called their \defword{range}) are maintained as beliefs. Finally, leaf nodes can be evaluated directly using the ranges, chance probabilities, and utilities (often more efficiently~\cite{12aamas-pcs}).

\subsubsection*{Imperfect Information Search, Decomposition, and Re-Solving}
\label{sec:bg-decomposition}

Solution concepts like Nash equilibria and minimax are defined over policy profiles. A player's policy is fixed during play and solely a function of the information state.
Search could instead be described as a process, which might return different action distributions at subsequent visits to the same state.  That is, when using search the resulting policy can depend on more than the just the information state, such as time-limits, non-deterministic computation, stochastic events from either the game or within the search, or the outcome of other searches.  These factors introduce important subtleties such as solution compatibility across different searches~\cite{sustr2020sound}.

CFR has been traditionally used as a game-solving engine, computing entire policies via self-play. Each iteration traverses the entire game tree or a sampled subtree, recursively computing the counterfactual values for an information state from the values of its successor states. Suppose one wanted a policy for a part of the game up to some depth $d > 0$. If there was an oracle to compute the counterfactual values each player would receive at depth $d$, then each iteration of CFR could be run to depth $d$ and query the oracle to return the values. As a result, the policies would not be available at depths $d' > d$. Summarizing the policies below depth $d$ by a set of values which can be used to reconstruct policies at depth $d$ and beyond is the basis of \defword{decomposition} in imperfect information games~\cite{Burch14CFRD}. 
A \defword{subgame} in an imperfect information game is a game rooted at a public state $\spub$. In order for a subgame to be a proper game, it is paired with a belief distribution $r$ over initial information states, $s \in \cS_i(\spub)$. This is a strict generalization of subgames in perfect information games, where every public state has exactly one information state (which is, in fact, no longer private as a result) and a singleton belief with probability $1$ for both players.

Subgame decomposition has been a crucial component of most recent developments of poker AI that scale to large games such as no-limit Texas hold'em~\cite{Moravcik17DeepStack,Brown17Libratus,Brown19Pluribus,brown2020combining}.
Subgame decomposition enables local search to refine the policy during play analogously to the classical search algorithms in perfect information games and traditional Bellman-style bootstrapping to learn value functions~\cite{Moravcik17DeepStack,serrino2019finding,Zarick20Supremus,brown2020combining}. Specifically, a \defword{counterfactual value network} (CVN) represented by parameters $\btheta$ encodes the value function $\bv_{\btheta}(\beta) = \{ v_i(s_i) \}_{s_i \in \cS_i(\spub), i \in \{1, 2\}}$, where $\beta$ includes player's beliefs over information states for the public information at $\spub$. The function $\bv_{\btheta}$ can then be used in place of the oracles mentioned above to summarize values of the subtrees below $\spub$.
An example of depth-limited CFR solving using decomposition is shown in Figure~\ref{fig:decomp}.

\defword{Safe re-solving} is a technique that generates subgame policies from only summary information of a previous (approximate) solution: a player's range and their opponent's counterfactual values. This is done by constructing an auxiliary game with specific constraints. The subgame policies in the auxiliary game are generated in a way that preserves the exploitability guarantees of the original solution, so they can replace the original policies in the subgame.  Thorough examples of the auxiliary game construction are found in~\cite{Burch14CFRD} and~\cite[Section 4.1]{Brown17Safe}.

\defword{Continual re-solving} is an analogue of classical game search, adapted to imperfect information games, that uses repeated applications of safe re-solving to play an episode of a game~\cite{Moravcik17DeepStack}. It starts by solving a depth-limited game tree rooted at the beginning of the game, and search is a re-solving step. As the game progresses, for every subsequent decision at some information state $s_i$, continual re-solving will refine the current strategy by re-solving at $s_i$. Like other search methods, it is using additional computation to more thoroughly explore a specific situation encountered by the player.

\subsection*{Related Work}
\label{sec:relwork}


\pog combines many elements that were originally proposed in AlphaZero and its predecessors, as well as DeepStack~\cite{Silver16Go,Silver2017AGZ,Silver18AlphaZero,Moravcik17DeepStack}. Specifically, \pog uses the combined search and learning using deep neural networks from AlphaGo and DeepStack, along with game-theoretic reasoning and search in imperfect information games from DeepStack. 
The use of public belief states and decomposition in imperfect information games has been a critical component of success in no-limit Texas Hold'em poker~\cite{Burch14CFRD,Brown17Safe,Moravcik17DeepStack,Brown17Libratus,Brown18,Brown19Pluribus,brown2020combining}.
The main difference from AlphaZero is that the search and self-play training in \pog are also sound for imperfect information games, and evaluation across game types. The main difference from DeepStack is the use of \CHANGED{substantially} less domain knowledge: the use of self-play (rather than poker-specific heuristics) to generate training data and a single network for all stages of the game.
The most closely related algorithm is Recurrent Belief-based Learning (ReBeL)~\cite{brown2020combining}. 
Like \pog, ReBeL combines search, learning, and game-theoretic reasoning via self-play.
The main difference is that \pog is based on (safe) continual re-solving and sound self-play. 
To achieve ReBeL's guarantees, its test-time search must be conducted with the same algorithm as in training, whereas \pog can use any belief-based value-and-policy network of the form described in \sectionref{sec:alg-cvpns} (similarly to e.g. AlphaZero, which trains using 800 simulations but then can use substantially larger simulation limits at test-time, \CHANGED{which is needed for strong performance in many perfect information domains).} \pog is also validated empirically across different challenge games of different game types, \CHANGED{whereas ReBeL is only evaluated on two imperfect information games.}

There has been considerable work in search for imperfect information games. One method that has been quite successful in practice is determinization: at decision-time, a set of of candidate world states are sampled, and some form of search is performed~\cite{CowlingISMCTS,Long10Understanding}. 
In fact, the baseline player we use to compare \pog to in Scotland Yard, PimBot~\cite{Nijssen12,PimThesis}, is based on these methods and achieved state-of-the-art results.
However, these methods are not guaranteed to converge to an optimal strategy over time.
We demonstrate this lack of convergence in practice over common search algorithms and standard \CHANGED{reinforcement learning (RL)} benchmarks in \supplementaryref{app:basic-rl-mcts-results}.
In contrast, the search in \pog is based on game-theoretic reasoning. Other algorithms have proposed adding game-theoretic reasoning to search: Smooth UCT~\cite{Heinrich2015SmoothUCT} combines Upper Confidence Bounds applied to Trees (UCT)~\cite{Kocsis06UCT} with fictitious play, however its convergence properties are not known. Online Outcome Sampling~\cite{Lisy15Online} derives an MCTS variant of Monte Carlo CFR~\cite{09nips-mccfr}; however, OOS is only guaranteed to approach an approximate equilibrium at a single information state (local consistency) and has not been evaluated in large games. GT-CFR used by \pog makes use of sound search based on decomposition and is globally consistent~\cite{Burch14CFRD,sustr2020sound}.

There have been a number of RL algorithms that have been proposed for two-player zero-sum games: Fictitious Self-Play~\cite{Heinrich15FSP}, Policy-Space Response Oracles (PSRO)~\cite{Lanctot17PSRO}, Double Neural CFR~\cite{Li19DNCFR}, Deep CFR and DREAM~\cite{Brown18DeepCFR,steinberger2020dream}, Regret Policy Gradients~\cite{Srinivasan18RPG}, Exploitability Descent~\cite{Lockhart19ED}, Neural Replicator Dynamics (NeuRD)~\cite{hennes2020neural}, Advantage Regret-Matching Actor Critic~\cite{gruslys2020advantage}, Friction FoReL~\cite{perolat2020poincare}, Extensive-form Double Oracle (XDO)~\cite{McAleer21XDO}, Neural Auto-curricula (NAC)~\cite{Feng21NAC},
and Regularized Nash Dynamics (R-NaD)~\cite{Perolat22Stratego}. These methods adapt classical algorithms for computing (approximate) Nash equilibria to the RL setting with sampled experience and general function approximation. As such, they combine game-theoretic reasoning and learning. Several of these methods have shown promise to scale: Pipeline PSRO defeated the best openly available agent in Stratego Barrage; ARMAC showed promising results on large poker games. 
R-NaD truly demonstrated scale by obtaining human-level performance in the very large game of Stratego~\cite{Perolat22Stratego}.
In Starcraft, AlphaStar was able to use human
data and game-theoretic reasoning to create a master-level real-time strategy policy~\cite{Silver19AlphaStar}.
However, none of them can use search at test-time to refine their policy; \CHANGED{this shifts a learning and function approximation burden onto training, typically making these methods more computationally demanding in both training time and model capacity to encode a policy or value function.}

Lastly, there have been works that use some combination of search, learning, and/or game-theoretic reasoning applied to specific domains. 
Neural networks have been trained via Q-learning to learn to play Scotland Yard~\cite{Dash2018SY}; however, the overall play strength of the resulting policy was not directly compared to any other known Scotland Yard agent.
In poker, Supremus proposed a number of improvements to DeepStack and demonstrated that they make a big difference when playing human experts~\cite{Zarick20Supremus}. Another work used a method inspired by DeepStack applied to The Resistance~\cite{serrino2019finding}. In the cooperative setting, several works have made use of belief-based learning (and search) using public subgame decomposition~\cite{foerster2019bayesian,lerer2020hanabi,sokota2021solving}, applied to Hanabi~\cite{Bard19Hanabi}.
Learning and game-theoretic reasoning were also recently combined to produce agents that play well with humans without human data on the collaborative game Overcooked~\cite{strouse2021collaborating}.
Search and reinforcement learning were combined to produce a bridge bidding player that cooperated with a state-of-the-art bot (WBridge5) and with humans~\cite{lockhart2020humanagent}.
Of considerable note is the game of (no-press) Diplomacy. In that game, game-theoretic reasoning was combined with learning in Best Response Policy Iteration~\cite{anthony2020learning}, and game-theoretic search and supervised learning were combined in~\cite{gray2020humanlevel} reaching human-level performance on the two-player variant. Recently, all three were combined in DORA~\cite{bakhtin2021nopress}, which learned to play Diplomacy without human data reaching human-level performance on the two-player variant, and subsequently Cicero~\cite{Bahkin22Cicero} reached human-level on the full game including communication with humans via language models. The main difference between \pog and these works is that they focus on specific games and exploit domain-specific knowledge to attain strong performance. 

\subsection*{Descriptions of Challenge Domains}

Chess and Go are well-known classic games, both seen as grand challenges of AI~\cite{DeepBlue,Gelly12Go}
that have driven progress in artificial intelligence since its inception. 
The achievement of DeepBlue beating Kasparov in 1997 is widely regarded to be the first big milestone of AI. Today, chess playing computer programs remain consistently super-human, and one of the strongest and most widely-used programs is Stockfish~\cite{Stockfish}.
Go emerged as the favorite new challenge domain, which was particularly difficult for classical search techniques~\cite{Gelly12Go}. 
Monte Carlo tree search~\cite{Coulom06MCTS,Kocsis06UCT,Browne12} emerged as the dominant search technique in Go. The best of these programs, Crazy Stone and Zen, were able to reach the level of 6 dan amateur~\cite{Silver16Go}. It was not until 2016 that AlphaGo defeated the first human professional Lee Sedol in the historical 2016 match, and also defeated the top human Ke Jie in 2017.

Heads-up no-limit Texas hold'em is the most common two-player version of poker played by humans, which is also played by DeepStack and Libratus~\cite{Moravcik17DeepStack,Brown17Libratus}. 
Human expert-level poker has been the standard challenge domain among imperfect 
information games, inspiring the field of game theory itself. No-limit Texas hold'em presents the complexity
of stochastic events (card draws), imperfect information (private cards), 
and a very large state space~\cite{johanson2013measuring}.
In this paper, we use blinds of 100 and 50 chips, and stack sizes of 200 big blinds (20,000 chips).

Scotland Yard is a compelling board game of imperfect information, receiving a Spiel des Jahres award in 1983 as well as being named the “The most popular game ’83” by SpielBox~\cite{syspiel}. 
The game is played on a map of London, where locations are connected by edges representing different modes of transportation. One player plays as ``Mr.~X'' (the evader) and others control detectives (pursuers). Mr.~X is only visible on specific rounds, but detectives get to see the mode of
transportation Mr.~X uses every round (e.g. taxi, bus, subway). In order to win, detectives need to catch Mr.~X within 24 rounds.  Scotland Yard is a perfect example of an imperfect information game that requires search for strong play---the detectives have to plan multiple moves into the future while reasoning about possible locations that Mr.~X may be.  Similarly, though Mr.~X has perfect information, he must also reason about where he could be to, for example, avoid revealing his location.  Unlike poker, Scotland Yard has partially-observable actions, so private information is effected by the agents' choices in addition to chance.

This suite of games covers the classic challenge domains across game types (perfect information and imperfect information, some with stochastic elements and others not), as well as an \CHANGED{additional} challenging imperfect information game with \CHANGED{substantially} longer sequences of actions and a fundamentally different type of uncertainty over hidden actions.

\section*{Results}
In order to understand the results, we give a brief \CHANGED{high-level} overview of our main algorithm, \poglong, which we present \CHANGED{formally} in the \sectionref{sec:mat_methods} section below.

\subsection*{\poglong: Algorithm Summary}

The \pog algorithm trains the agent via \defword{sound self-play}: each player, when faced with a decision to make,
employs a sound \defword{growing-tree CFR (GT-CFR)} search equipped with a \defword{counterfactual value-and-policy network} (CVPN) to generate a policy for the current state, which is then used to sample an action to take. 

GT-CFR grows a tree, starting with the current public state, and consists of two alternating phases: the \emphword{regret update phase} runs public tree CFR updates on the current tree; 
the \emphword{expansion phase} expands the tree by adding new public states via simulation-based expansion trajectories.
One iteration of \gtcfr{} consists of one run of the regret update phase followed by one run of the expansion phase.

The self-play process generates two types of training data for updating the value and policy networks: \emphword{search queries}, which are public belief states that were queried by the CVPN during the \gtcfr{} regret update phase, and \emphword{full-game trajectories} from the self-play games. The search queries must be \emphword{solved} to compute counterfactual value targets for updating the value network. The full-game trajectories provide targets for updating the policy network.  In practice, the self-play data generation and training happen in parallel: actors generate the self-play data (and solve queries) while trainers learn new networks and periodically update the actors. 

\subsection*{Theoretical Results}
We have two main theoretical results, which we describe here only informally. They are formally treated in \sectionref{sec:mat_methods}.
Theorem~\ref{thm:main-theorem} ensures that the exploitability of the final \gtcfr{} policy is at most $O(1/\sqrt{T})$, where $T$ is number of \gtcfr{} iterations, under some conditions on how the search tree is expanded and so long as the value function is reasonably accurate.  \pog invokes \gtcfr{} to re-solve a subtree every time it must act. Thereom~\ref{thm:continual-resolving} bounds the exploitability of the entire \pog policy proving that it is sound to employ \gtcfr{} recursively.
Both theorems together ensure that the search is sound up to some acceptable error in the value function. If there is no error in the value function, and the values are the game-theoretic optimal values, then \gtcfr{} will provably converge to a Nash equilibrium strategy if run under the conditions stated in the theorems.

\subsection*{Experimental Results}

We evaluate \pog on four games: chess, Go, heads-up no-limit Texas hold'em poker, and Scotland Yard. We also evaluate \pog on the commonly-used small benchmark poker game Leduc hold'em, and a custom-made small Scotland Yard map, where the approximation quality compared to the optimal policy can be computed exactly.

When reporting the results we use the notation $\pog(s, c)$ for \pog running GT-CFR with $s$ total expansion simulations, and $c$ expansion simulations per regret update phase, so the total number of GT-CFR iterations is then $\frac{s}{\lceil c \rceil}$.
For example, $\pog(8000, 10)$ refers to 8000 expansion simulations at 10 expansions per regret update (800 GT-CFR iterations).
We choose this notation style to be easily comparable to number of simulations in AlphaZero.

\subsubsection*{Exploitability in Leduc Poker and Small Scotland Yard Map}

Supporting our theoretical results, we empirically evaluate the exploitability of \pog in Leduc poker~\cite{Southey05bayes} and in Scotland Yard on a small map named ``glasses''. The full description of Leduc poker is presented in \supplementaryref{app:additional-results} and the map is illustrated in Figure~\ref{fig:scotlandyard-glasses}.

Exploitability is a function of a specific (fixed) policy profile. However, for a search algorithm like \pog, previous searches may affect policies computed at later points within the same game, as explained in \sectionref{sec:bg-decomposition}. Hence, we construct multiple samples of the \pog policy by choosing a random seed, running the search algorithm at every public state in a breadth-first manner such that every search is conditioned on previous searches at predecessor states, and composing together the policies obtained from each search. We then show the minimum, average, and maximum exploitabilities over policies constructed in this way from 50 different choices of seeds. If the minimum and maximum exploitability values are tight, then they represent an accurate estimate of true exploitability.

Figure~\ref{fig:exploitability} shows the exploitability of \pog in Leduc poker and the glasses map of Scotland Yard, as a function of the number of CVPN training steps.
For these graphs, we evaluate multiple networks (each trained for a different number of steps) generated by a single training run of $\pog(100, 1)$. Each data point corresponds to a specific network (determined by number of steps trained) being evaluated under different settings during play.
For each specific x-value, a single network was used to obtain each exploitability value of \pog using the network under different evaluation conditions.

We observe that exploitability drops fairly quickly as the training steps increase. Also, even using only 1 CFR update per simulation, there is significant difference in exploitability when more simulations are used.
As Theorem~\ref{thm:main-theorem} suggests, more training (by reducing $\epsilon$) and more search (by increasing $T$) reduces the exploitability of \pog. Standard RL algorithms in self-play are not guaranteed to reduce exploitability with continued training in this setting. We show this lack of convergence in practice in \supplementaryref{app:basic-rl-mcts-results}.

\subsubsection*{Results in Challenge Domains}

Our main results compare the performance of \pog to other agents in our challenge domains.
We trained a version of AlphaZero using its original settings in chess and Go, \eg, using 800 MCTS simulations during training, with 3500 concurrent actors each on a single TPUv4, for a total of 800k training steps. \pog was trained using a similar amount of TPU resources.

In chess, we evaluated~\pog against Stockfish 8 level 20~\cite{Stockfish} and AlphaZero.
$\pog(400,1)$ was run in training for 3M training steps.
During evaluation, Stockfish uses various search controls: number of threads, and time per search.
We evaluate AlphaZero and \pog up to 60,000 simulations. 
A tournament between all of the agents was played at 200 games per pair of agents (100 games as white, 100 games as black).
\CHANGED{From this tournament, we rank players according to their Elo ratings.
Elo is a classic system for rating chess players originally designed by Arpad Elo in 1967 and still widely-used today~\cite{Elo1986rating} in many games. A rating, $r_i$,  is  assigned  to  each player $i$ such that a logistic model predicts the probability of player $i$ beating player $j$ as $1/(1 + 10^{(r_j - r_i)/400})$.}
Table~\ref{tab:elo-results} shows the relative Elo comparison obtained by this tournament, where a baseline of 0 is chosen for Stockfish(threads=1, time=0.1s).

In Go, we evaluate $\pog(60000, 10)$ using a similar tournament as in chess, against two previous Go programs: GnuGo (at its highest level, 10)~\cite{gnugo} and Pachi v7.0.0~\cite{pachi} with 10k and 100k simulations, as well as AlphaZero~\cite{Silver18AlphaZero} with a range of search simulations at different points in training.
$\pog(400,1)$ was used in training for 1M training steps.
Table~\ref{tab:elo-results} shows the relative Elo comparison for a subset of the agents that played in this tournament, where a baseline of 0 is chosen for GnuGo. The full results are presented in Tables~\ref{tab:full-go-results} and \ref{tab:full-go-results-recursive}.

Notice in both chess and Go that \pog reaches strong performance. In chess, $\pog(60000,10)$ is stronger than Stockfish using 4 threads and one second of search time. In Go, $\pog(16000, 10)$ is more than 1100 Elo stronger than Pachi with 100,000 simulations. Also, $\pog(16000, 10)$ wins 0.5\% (2/400) of its games against AlphaZero(s=8000,t=800k). As a result, \pog appears to be performing at the level of top human amateur, possibly even professional level. In both cases, \pog is weaker than AlphaZero, with the gap being smaller in chess. We hypothesize that this difference is the result of MCTS being more efficient than CFR on perfect information games, as the price of \pog's generality.

For chess and Go, we also present direct Elo comparisons from a tournament between AlphaZero (trained for 800k steps) and \pog agents when increasing the number of neural network evaluations in Figure~\ref{fig:pog-and-az-scaling}. 
These results demonstrate that \pog is able to scale, improving performance with available computation.  Note that while the neural networks evaluations account for the majority of the run time, the complexity of the regret update phase is linear in the size of the tree. The run time is thus quadratic in the number of GT-CFR iterations. The absolute time cost could be reduced by an implementation that runs the regret update and expansion phase in parallel. For a more detailed analysis of \pog's complexity, see \sectionref{ref:alg-guarantees}.
Intuitively, we would expect $c = 1$ (corresponding to one regret update per expansion simulation) to be best choice. Due to these computational constraints, we chose by hand a small number of values for $c > 1$. Interestingly, we did notice that $c = 1$ is not always the best choice in practice and hope to explore this more thoroughly in the future.

In heads-up no-limit Texas hold'em, we evaluate \pog against Slumbot2019~\cite{Jackson13SlumbotNL,JacksonSlumbotGithub}, the best open-source heads-up no-limit computer poker player.
When training poker, \pog uses randomized betting abstractions described in \supplementaryref{sec:poker-bets} to reduce the number of actions from 20,000 to 4 or 5.
$\pog(10, 0.01)$ is trained for up to 1.1M training steps and then evaluated.
Since poker has particularly high variance, we use the Action-Informed Value Assessment Tool (AIVAT)~\cite{burch2017aivat} to compute a more accurate estimate of performance. 
We also evaluate \pog against a local best-response (LBR) player that can use only fold and call actions with a poker-specific heuristic, which has shown to find exploits in previous poker agents~\cite{lisy2017eqilibrium}. Table~\ref{tab:hunl} summarizes the results of \pog along with other recent poker agents. $\pog(10, 0.01)$ wins on average $7 \pm 3$ milli big blinds (0.7 chips) per hand, with 95\% confidence intervals (3.1M matches). LBR fails to find an exploit of \pog's strategy, and \pog wins on average by 434 $\pm$ 9 milli big blinds per hand.


In Scotland Yard, the current state-of-the-art agent in this game is based on MCTS with game-specific heuristic
enhancements~\cite{Nijssen12}. 
We call this agent ``PimBot'' based on its main author, Joseph Antonius Maria (``Pim'') Nijssen.
PimBot implements a variant of MCTS that uses determinization, heuristic evaluations and playout 
policies~\cite{Nijssen12,PimThesis}. PimBot won 34 out of 50 manually played games against the Nintendo DS Scotland Yard AI.

In our experiment \pog is trained up to 17M steps. In evaluation we play a head-to-head match with $\pog(400,1)$ against PimBot at different number of simulations per search. The results are shown in Figure~\ref{fig:pog_x_pimbot}. These results show that \pog is winning significantly even against PimBot with 10M search simulations (55\% win rate), compared to \pog searching a tiny fraction of the game. Interestingly PimBot does not seem to play stronger with more search at this point, as both the 1M and 10M iteration versions have the same performance against \pog.

As in chess and Go, \pog also demonstrates strong performance in these complex imperfect information games. In the case of poker, in addition to beating Slumbot it also beats the local best-response agent which was not possible for some previous agents (including Slumbot). Finally, \pog significantly beats the state-of-the-art agent in Scotland Yard, an imperfect information game with longer episodes and fundamentally different kind of imperfect information than in poker. Together, these results indicate that \pog is capable of strong performance across four games, two fundamentally different game types, and can act as a truly unified algorithm combining search, learning, and game-theoretic reasoning for competitive games.



\section*{Discussion}

\poglong (\pog) is a unified algorithm that combines search, learning, and game-theoretic reasoning. 
\pog is comprised of two main components: a growing-tree counterfactual regret minimization (GT-CFR) technique, and sound self-play which learns counterfactual value-and-policy networks via self-play. Most notably, \pog is a sound algorithm for both perfect and imperfect information games: as computational resources increase, \pog is guaranteed to produce better approximation of minimax-optimal strategies. This finding is also verified empirically in Leduc poker, where additional search leads to test-time approximation refinement, unlike any pure reinforcement learning algorithms that do not use search.

\CHANGED{In addition to being sound, \pog also demonstrates}
strong performance on challenge domains, using minimal domain knowledge. In the perfect information games of chess and Go, \pog performs at the level of human experts or professionals, but can be \CHANGED{substantially} weaker in head-to-head play than specialized algorithms for this class of games, like AlphaZero, when given the same resources. In the imperfect information game no-limit Texas hold'em poker, \pog beats Slumbot, the best openly available poker agent, and is shown not to be exploited by a local best-response agent using poker-specific heuristics. In Scotland Yard, \pog defeats the state-of-the-art agent.

There are some limitations of \pog that are worth investigating in future work. First, the use of betting abstractions in poker could be removed in favor of a general action-reduction policy for large action spaces. Second, \pog currently requires enumerating the information states per public state, which can be prohibitively expensive in some games; this might be approximated by a generative model that samples world states and operates on the sampled subset. Finally, substantial computational resources are used to attain strong play in challenge domains; an interesting question is whether this level of play is achievable with less computational resources.

\section*{Materials and Methods}
\label{sec:mat_methods}

We now give a detailed description of the \poglong algorithm. As \pog has several components, we describe them each individually first, and then describe how they are all combined toward the end of the section.
For clarity, many of the details (including full pseudocode) are presented in \supplementaryref{sec:alg-details}.

\subsection*{Counterfactual Value-and-Policy Networks}
\label{sec:alg-cvpns}

The first major component of \pog is a \defword{counterfactual value-and-policy network} (CVPN) with parameters $\btheta$, depicted in Figure~\ref{fig:cvpn}. These parameters represent a function $f_{\btheta}(\beta) = (\bv, \bp)$, where outputs $\bv$ are counterfactual values (one per information state per player), and \emphword{prior policies} $\bp$, one per information state for the acting player, in the public state $\spub(h)$ at some history of play $h$.

In our experiments, we use standard feed-forward networks and residual networks. The details of the architecture are described in \supplementaryref{sec:net-arch}.

\subsection*{Search via Growing-Tree CFR}
\label{sec:alg-gtcfr}

Growing-tree CFR (GT-CFR) is an algorithm that runs a CFR variant on a public game tree that is incrementally grown over time.
GT-CFR starts with an initial tree, $\cL^0$, containing $\beta$ and all of its child public states. Then each iteration, $t$, of GT-CFR consists of two phases:
\begin{enumerate}
    \item The \defword{regret update phase} runs several public tree CFR updates on the current tree $\cL^t$. 
    \item The \defword{expansion phase} expands $\cL^t$ by adding new public states via simulation-based expansion trajectories, producing a new larger tree $\cL^{t+1}$.
\end{enumerate}

\noindent When reporting the results we use the notation $\pog(s, c)$ for \pog running GT-CFR with $s$ total expansion simulations, and $c$ expansion simulations per regret update phase, so the total number of GT-CFR iterations is then $\frac{s}{\lceil c \rceil}$. The $c$ can be fractional, so \eg 0.1 indicates a new node every 10 regret update phases. Figure~\ref{fig:pog-search} depicts the whole GT-CFR cycle. We chose this specific notation to directly compare total expansion simulations, $s$, to AlphaZero.

The regret update phase runs $\left \lceil \frac{1}{c} \right \rceil$ updates (iterations) of public tree CFR on $\cL^t$ using simultaneous updates, regret-matching$^+$, and linearly-weighted policy averaging~\cite{Tammelin15CFRPlus}. 
At public tree leaf nodes, a \defword{query} is made to the CVPN at belief state $\beta'$, whose values $f_{\btheta}(\beta') = (\bv, \bp)$ are used as estimates of counterfactual values for the public subgame rooted at $\beta'$. 

In the expansion phase, new public tree nodes are added to $\cL$.
Search statistics, initially empty, are maintained over information states $s_i$, accumulated over all expansion phases within the same search.
At the start of each simulation, an information state $s_i$ is sampled from the beliefs in $\beta_{\text{root}}$. Then, a world state $w_{\text{root}}$ is sampled from $s_i$, with associated history $h_{\text{root}}$.
Actions are selected according to a mixed policy that takes into account learned values (via $\pi_{\text{PUCT}}(s_i(h))$) as well as the currently active policy ($\pi_{\text{CFR}}(s_i(h))$) from search: $\pi_{\text{select}}(s_i(h)) = \frac{1}{2} \pi_{\text{PUCT}}(s_i(h)) + \frac{1}{2} \pi_{\text{CFR}}(s_i(h))$. The first policy is determined by PUCT~\cite{Silver16Go} using counterfactual values $v_i(s_i, a)$ normalized by the sum of the opponent's reach probability at $s_i$ to resemble state-conditional action values, and the prior policy $\bp$ obtained from the queries.
The second is simply CFR's average policy at $s_i(h)$. 
As soon as the simulation encounters an information state $s_i \in \spub$ such that $\spub \not\in \cL$, the simulation ends, $\spub$ is added to $\cL$, and visit counts are updated along nodes visited during the trajectory.
Similarly to AlphaZero~\cite{Silver18AlphaZero}, virtual losses~\cite{SegalVirtualLoss} are added to the PUCT statistics when doing $\lceil c \rceil$ simulations inside one GT-CFR iteration.


AlphaZero always expands a single action/node at the end of the iteration (the action with the highest UCB score).
Optimal policies in perfect information games can be deterministic, and expanding a single action/node is a good way to avoid unneeded computation after unpromising actions.
MCTS methods are sound as long as the best action has been added, which is always true in the limit as the tree is completely filled out.
In imperfect information games, optimal policies might be stochastic, having non-zero probability over multiple actions.
Rather than expanding a single action, \pog thus expands the top $k$ actions as ranked by the prior.
We use $k=1$ for perfect information games, where computation cost is very important and we only need to find a single good action, and $k=\infty$ to add all children for imperfect information games where it is important to mix over multiple actions.
In addition to being sound in the limit, \pog also has a finite-time guarantee on policy quality when $k=\infty$.

\subsection*{Modified Continual Re-solving}
\label{sec:modified-resolving}

The continual re-solving method used by DeepStack~\cite{Moravcik17DeepStack} takes advantage of a few poker properties, which are not found in other games like Scotland Yard, so we use a more general re-solving method that can be applied to a broader class of games.
Recall that the re-solving step and the corresponding auxiliary game requires i) the current player's range ii) the opponent's counterfactual values.
This provides a succinct and sufficient representation to safely re-solve the subgame rooted in a public state $s_{pub}$.
In hold'em poker, players generally take turns making actions, and a depth-limited search tree for a re-solving auxiliary game can always be deep enough to contain a state for the opponent's action.
All player actions are fully visible to both players, so the opponent's maximum counterfactual value in the previous search tree can be used for the next re-solving auxiliary game, no matter what opponent action we are responding to.
By working within the single, fixed domain of poker, these properties let DeepStack and Libratus simply retrieve the re-solving summary information from its previous search.

As \pog is a general algorithm, it can no longer leverage this special case.
The current public state $s_\text{pub}$ might not have been included in the previous search tree, so the prior computation might not directly provide us with the required summary information for re-solving the subgame rooted in $s_\text{pub}$.
\pog thus starts its re-solving process in the state closest to the current state that is included in the previous search tree: $s^{prev}_{pub}$.
We initialize the search tree with a single branch leading from $s^\text{prev}_\text{pub}$ to $s_\text{pub}$, with all off-branch actions being leaves.
The search tree is then expanded forward from $s_\text{pub}$, as in DeepStack or Libratus.
This re-solving auxiliary game uses summary information for $s^\text{prev}_\text{pub}$ instead of $s_\text{pub}$, and by construction these values and probabilities are available in the previous search tree.

When generating a policy for this next re-solving auxiliary game with GT-CFR, we constrain the expansion phase of GT-CFR to only grow the tree under $s_\text{pub}$, to focus the computation on the states relevant for the current decision.
After re-solving, our action probabilities for our current information state will still come from the new re-solved policy at $s_\text{pub}$, which might not be at the root of the search tree.

Finally, like DeepStack, the gadget for the re-solving auxiliary game is modified by mixing in the opponent's range from the previous search.
As introduced in \cite{Burch14CFRD}, the gadget used to transform a subgame into a re-solving auxiliary game is a binary opponent decision for each opponent information state before the subgame.
At each information state, the opponent can either terminate (T) and receive the opponent counterfactual values in the re-solving summary, or follow (F) this line of play into the corresponding subgame.
The effect of the gadget is to generate an opponent range $r$ for the subgame.
Given an opponent range $r^{\text{prev}}$ from the previous search, \cite{Moravcik17DeepStack} modified the opponent range to be $\alpha r + (1-\alpha) r^\text{prev}$.
As with DeepStack, this regularization towards the previous opponent policy empirically improves the performance, and we used $\alpha = 0.5$.


\subsubsection*{Performance Guarantees for Continual Re-solving}
\label{ref:alg-guarantees}

Growing the tree in GT-CFR allows the search to selectively focus on parts of the space that are important for local decisions. Starting with a small tree and adding nodes over time does not have an additional cost in terms of convergence:

\begin{theorem}
\label{thm:main-theorem}
Let $\mathcal{L}^t$ be the public tree at time $t$.
Assume public states are never removed from the search tree, so $\mathcal{L}^t \subseteq \mathcal{L}^{t+1}$.
For any given tree $\mathcal{L}$, let $\treeinterior{\mathcal{L}}$ be the interior of the tree: all non-leaf, non-terminal public states where \gtcfr{} generates a policy.
Let $\treefrontier{\mathcal{L}}$ be the frontier of $\mathcal{L}$, containing the non-terminal leaves where \gtcfr{} uses $\epsilon$-noisy estimates of counterfactual values. 
Let $U$ be the maximum difference in counterfactual value between any two strategies, at any information state, and $A$ be the maximum number of actions at any information state.
Then, the regret at iteration $T$ for player $i$ is bounded:
\begin{align*}
    R^{T, \text{full}}_i \le \sum_{t=1}^T |\treefrontier{\mathcal{L}^t}|\epsilon + \sum_{\pubstate{} \in \treeinterior{\mathcal{L}^T}} |\mathcal{S}_i(\pubstate{})| U \sqrt{AT}
\end{align*}
\end{theorem}

The regret $R^{T, \text{full}}$ in Theorem~\ref{thm:main-theorem} is the gap in performance between \gtcfr{} iterations and the highest-value strategy. Theorem~\ref{thm:main-theorem} shows that the average policy returned by \gtcfr{} converges towards a Nash equilibrium at a rate of $1/\sqrt{T}$, but with some minimum exploitability due to $\epsilon$-error in the value function.
There is also no additional cost when using \gtcfr{} as the game-solving algorithm for each re-solving search step in continual re-solving:
\begin{theorem}
\label{thm:continual-resolving}
Assume we have played a game using continual re-solving, with one initial solve and $D$ re-solving steps.
Each solving or re-solving step finds an approximate Nash equilibrium through $T$ iterations of \gtcfr{} using an $\epsilon$-noisy value function, public states are never removed from the search tree, the maximum interior size $\sum_{\pubstate{} \in \treeinterior{\mathcal{L}^T}} |\mathcal{S}_i(\pubstate)|$ of the tree is always bounded by $N$, the frontier size of the tree is always bounded by $F$, the maximum number of actions at any information states is $A$, and the maximum difference in values between any two strategies is $U$.
The exploitability of the final strategy is then bounded by
$(5D+2)\left(F\epsilon + NU \sqrt{\frac{A}{T}}\right)$.
\end{theorem}
Theorem~\ref{thm:continual-resolving} is similar to Theorem 1 of \cite{Moravcik17DeepStack}, adapted to \gtcfr{} and using a more detailed error model which can more accurately describe value functions trained on approximate equilibrium strategies. It shows that continual re-solving with \gtcfr{} has the general properties we might desire: exploitability decreases with more computation time and decreasing value function error, and only increases linearly with game length.
Proofs of these theorems are presented as supplementary text.

The computational complexity of a \gtcfr{} re-solving step with $T$ iterations expanding $k$ children is $\mathcal{O}(kT^2)$ public states visited and CVPN network calls.
In the special case of perfect information games, the number of network calls can be reduced to $\mathcal{O}(T)$.
At every iteration $t$, in the expansion phase \gtcfr{} will traverse a single trajectory through the tree to expand a leaf, and use the CVPN to evaluate the newly expanded children.
This requires $k$ network calls, and a worst case of $|\mathcal{L}^t|$ states visited or $\lceil \log_b{|\mathcal{L}^t|} \rceil$ states visited in a balanced $b$-ary tree.
In the regret update phase, \gtcfr{} visits every state in $\mathcal{L}^t$ and uses the CVPN to evaluate every leaf of the tree.
Because $k$ child states are added to the tree at each iteration, $|\mathcal{L}^t| \le \mathcal{O}(kt)$, giving the stated bounds.

In perfect information games, $k=1$, each player range is a single number, and we only need to evaluate a state once because the optimal policy does not depend on the player ranges.
If a state is evaluated once with a range of $1$ for both players and then stored, any other belief state can by evaluated by scaling the stored result by the opponent's ranges.

\subsection*{Data Generation via Sound Self-play}
\label{sec:alg-selfplay}

\poglong generates episodes of data in self-play by running searches at each decision point.
Each episode starts at the initial history $h_0$ corresponding to the start of the game, and produces a sequence of histories $(h_0, h_1, \cdots)$. At time $t$, the agent runs a local search and then selects an action $a_t$,
and the next history $h_{t+1}$ is obtained from the environment by taking action $a_t$ at $h_t$.
Data for training the CVPN is collected via resulting trajectories and the individual searches.

When generating data for training the CVPN, it is important that searches performed at different public states be consistent with both the CVPN represented by $\btheta$ and with searches made at previous public states along the same trajectory (\eg two searches should not be computing parts of two different optimal policies). This is a critical requirement for sound search~\cite{Burch14CFRD,Moravcik17DeepStack,sustr2020sound}, and we refer to the process of a sound search algorithm generating data in self-play as \defword{sound self-play}.
To achieve sound self-play, searches performed during data generation run GT-CFR on the modified safe re-solving auxiliary game (as described in \sectionref{sec:modified-resolving}).

\subsection*{Training Process}
\label{sec:training}

The quality of the policies produced by GT-CFR and data generated by sound self-play depends critically on the values returned by the CVPN. 
Hence, it is important for the estimates to be accurate in order to produce high-performance searches and generate high-quality data.
In this subsection, we describe the procedure we use to train the CVPN. The process is summarized in Figure~\ref{fig:pog-training}.

\subsubsection*{Query Collection}
\label{sec:training-query-collection}

As described in \sectionref{sec:alg-gtcfr} and \sectionref{sec:alg-selfplay},
episodes are generated by each player running searches of GT-CFR from the current public state. Each search produces a number of network queries from public tree leaf nodes $\beta$ (depicted as pink nodes in Figure~\ref{fig:pog-training}). 

The training process improves the CVPN via supervised learning.  Values are trained using Huber loss~\cite{huber64} based on value targets and the policy loss is cross entropy with respect to a target policy. 
Value and policy targets are added to a sliding window data set of training data that is used to train the CVPN concurrently. The CVPN is updated asynchronously on the actors during training.

\subsubsection*{Computing Training Targets}
\label{sec:training-targets}

Policy targets are assembled from the searches started at public states along the main line of episodes (the histories reached in self-play) generated by sound self-play described in \sectionref{sec:alg-selfplay}. 
Specifically, they are the output policies for all information states within the root public state, computed in the regret update phase of GT-CFR.

Value targets are obtained in two different ways. Firstly, the outcome of the game is used as a (TD(1)) value target for states along the main line of episodes generated by sound self-play.
Secondly, value targets are also obtained
by bootstrapping: running an instance of GT-CFR from subgames rooted at input queries. In principle, any solver could be used because any subgame rooted at $\beta$ has well-defined values. Thus, this step acts much like a policy improvement operator via decomposition described in \sectionref{sec:bg-decomposition}.
Specifically, the value targets are the final counterfactual values after $T$ iterations of GT-CFR for all the information states within the public state that initiated the search.
The specific way that the different value targets are assigned is described by the pseudocode in \supplementaryref{sec:pseudocode} and determined by a hyperparameter noted in Table~\ref{tab:hparams-per-game}.

\subsubsection*{Recursive Queries}
\label{sec:recursive-queries}

While the solver is computing targets for a query, it is also generating more queries itself by running GT-CFR.
Some of these \defword{recursive queries} are also added to the buffer for future solving, so that the CVPN can produce reasonable answers for all leaves in a search, not just those on the self-play lines.
As a result, at any given time the buffer may include queries generated by search in the main self-play game or by solver-generated queries off the main line.
To ensure that the buffer is not dominated by recursive queries, we set the probability of adding a new recursive query to less than 1 (in our experiments, the value is typically 0.1 or 0.2; see Table~\ref{tab:hparams-per-game} for the exact values).

\subsubsection*{Consistency of Training Process}

One natural question is whether, or under what circumstances, the training process could ensure convergence to the optimal values? 
The answer is positive: the training process converges to the optimal values, asymptotically, as $T \rightarrow \infty$ and with very large (exponential) memory.

Informally, imagine an oracle function $f(\beta)$ that can simply memorize the values and policy for the particular $\beta$ similar to a tabular value or policy iteration algorithm except with continuous keys. For any subgame rooted at some $\beta$ with a depth of 1 (every action leads to terminal states), the values and policies can be computed and stored for $\beta$ after $T$ iterations of the solver. This can then be applied inductively: since CFR is deterministic, for any subgame on the first iteration of GT-CFR, a finite number of queries will be generated. Each of these queries will be solved using GT-CFR. Eventually, the query will be a specific one that is one step from the terminal state whose values can be computed exactly and stored in $f(\beta)$. As this value was generated in self-play or by a query solver, and CFR is deterministic, it will produce another self-play game with the identical query, except it will load the solved value from $f(\beta)$, and inductively the values will get propagated from the bottom up.
Since CFR is deterministic and $T$ is finite, these ensure that the memory requirement is not infinite despite the continuous-valued keys. Practically, the success of the training process will depend on the representational capacity and training efficacy of the function approximation (\ie neural network architecture).

For a fully-detailed description of the algorithm, including hyperparameter values and specific descriptions of each process described above, see \supplementaryref{sec:alg-details}.

%


\ifcamera
\else
\bibliography{paper}

\begin{thebibliography}{10}

\bibitem{Samuel59}
A.~L. Samuel, Some studies in machine learning using the game of checkers.
\newblock {\it IBM Journal of Research and Development\/} {\bf 44}, 206--226
  (1959).

\bibitem{RussellNorvig2010}
S.~J. Russell, P.~Norvig, {\it Artificial Intelligence: A Modern Approach\/}
  (Pearson Education, 2010), third edn.

\bibitem{Silver16Go}
D.~Silver, A.~Huang, C.~J. Maddison, A.~Guez, L.~Sifre, G.~van~den Driessche,
  J.~Schrittwieser, I.~Antonoglou, V.~Panneershelvam, M.~Lanctot, S.~Dieleman,
  D.~Grewe, J.~Nham, N.~Kalchbrenner, I.~Sutskever, T.~Lillicrap, M.~Leach,
  K.~Kavukcuoglu, T.~Graepel, D.~Hassabis, Mastering the game of {G}o with deep
  neural networks and tree search.
\newblock {\it Nature\/} {\bf 529}, 484--489 (2016).

\bibitem{DeepBlue}
M.~Campbell, A.~J. Hoane, F.-h. Hsu, Deep blue.
\newblock {\it Artificial Intelligence\/} {\bf 134}, 57–83 (2002).

\bibitem{TDGammon}
G.~Tesauro, {TD}-{G}ammon, a self-teaching backgammon program, achieves
  master-level play.
\newblock {\it Neural Comput.\/} {\bf 6}, 215–219 (1994).

\bibitem{Silver18AlphaZero}
D.~Silver, T.~Hubert, J.~Schrittwieser, I.~Antonoglou, M.~Lai, A.~Guez,
  M.~Lanctot, L.~Sifre, D.~Kumaran, T.~Graepel, T.~Lillicrap, K.~Simonyan,
  D.~Hassabis, A general reinforcement learning algorithm that masters chess,
  shogi, and {G}o through self-play.
\newblock {\it Science\/} {\bf 632}, 1140--1144 (2018).

\bibitem{Johanson16phdthesis}
M.~B. Johanson, Robust strategies and counter-strategies: From superhuman to
  optimal play, Ph.D. thesis, University of Alberta (2016).
  \url{http://johanson.ca/publications/theses/2016-johanson-phd-thesis/2016-johanson-phd-thesis.pdf}.

\bibitem{Moravcik17DeepStack}
M.~Morav{\v c}{\'\i}k, M.~Schmid, N.~Burch, V.~Lis{\'y}, D.~Morrill, N.~Bard,
  T.~Davis, K.~Waugh, M.~Johanson, M.~Bowling, Deepstack: Expert-level
  artificial intelligence in heads-up no-limit poker.
\newblock {\it Science\/} {\bf 358} (2017).

\bibitem{Brown17Libratus}
N.~Brown, T.~Sandholm, Superhuman {AI} for heads-up no-limit poker: {L}ibratus
  beats top professionals.
\newblock {\it Science\/} {\bf 360} (2017).

\bibitem{Brown19Pluribus}
N.~Brown, T.~Sandholm, Superhuman {AI} for multiplayer poker.
\newblock {\it Science\/} {\bf 365}, 885--890 (2019).

\bibitem{Bard19Hanabi}
N.~Bard, J.~N. Foerster, S.~Chandar, N.~Burch, M.~Lanctot, H.~F. Song,
  E.~Parisotto, V.~Dumoulin, S.~Moitra, E.~Hughes, I.~Dunning, S.~Mourad,
  H.~Larochelle, M.~G. Bellemare, M.~Bowling, The {Hanabi} challenge: A new
  frontier for {AI} research.
\newblock {\it Artificial Intelligence\/} {\bf 280} (2020).

\bibitem{lerer2020hanabi}
A.~Lerer, H.~Hu, J.~Foerster, N.~Brown, Improving policies via search in
  cooperative partially observable games.
\newblock {\it Proceedings of the Thirty-Fourth AAAI Conference on Artificial
  Intelligence\/} ({AAAI}, 2020).

\bibitem{serrino2019finding}
J.~Serrino, M.~Kleiman-Weiner, D.~C. Parkes, J.~B. Tenenbaum, Finding friend
  and foe in multi-agent games.
\newblock {\it Proceedings of the Thirty-third Conference on Neural Information
  Processing Systems\/} ({NeurIPS}, 2019).

\bibitem{lockhart2020humanagent}
E.~Lockhart, N.~Burch, N.~Bard, S.~Borgeaud, T.~Eccles, L.~Smaira, R.~Smith,
  Human-agent cooperation in bridge bidding.
\newblock {\it Proceedings of the Cooperative AI Workshop at 34th Conference on
  Neural Information Processing Systems\/} ({NeurIPS}, 2020).

\bibitem{Silver19AlphaStar}
O.~Vinyals, I.~Babuschkin, W.~M. Czarnecki, M.~Mathieu, A.~Dudzik, J.~Chung,
  D.~H. Choi, R.~Powell, T.~Ewalds, P.~Georgiev, J.~Oh, D.~Horgan, M.~Kroiss,
  I.~Danihelka, A.~Huang, L.~Sifre, T.~Cai, J.~P. Agapiou, M.~Jaderberg, A.~S.
  Vezhnevets, R.~Leblond, T.~Pohlen, V.~Dalibard, D.~Budden, Y.~Sulsky,
  J.~Molloy, T.~L. Paine, C.~Gulcehre, Z.~Wang, T.~Pfaff, Y.~Wu, R.~Ring,
  D.~Yogatama, D.~W{\"u}nsch, K.~McKinney, O.~Smith, T.~Schaul, T.~Lillicrap,
  K.~Kavukcuoglu, D.~Hassabis, C.~Apps, D.~Silver, Grandmaster level in
  {S}tar{C}raft {II} using multi-agent reinforcement learning.
\newblock {\it Nature\/} {\bf 575}, 350--354 (2019).

\bibitem{anthony2020learning}
T.~W. Anthony, T.~Eccles, A.~Tacchetti, J.~Kram{\'{a}}r, I.~M. Gemp, T.~C.
  Hudson, N.~Porcel, M.~Lanctot, J.~P{\'{e}}rolat, R.~Everett, S.~Singh,
  T.~Graepel, Y.~Bachrach, Learning to play no-press {D}iplomacy with best
  response policy iteration.
\newblock {\it Thirty-third Conference on Neural Information Processing
  Systems\/} ({NeurIPS}, 2020).

\bibitem{gray2020humanlevel}
J.~Gray, A.~Lerer, A.~Bakhtin, N.~Brown, Human-level performance in no-press
  {D}iplomacy via equilibrium search.
\newblock {\it In Proceedings of the International Conference on Learning
  Representations\/} ({ICLR}, 2020).

\bibitem{bakhtin2021nopress}
A.~Bakhtin, D.~Wu, A.~Lerer, N.~Brown, No-press {D}iplomacy from scratch.
\newblock {\it Proceedings of the Thirty-fourth Conference on Neural
  Information Processing Systems\/} ({NeurIPS}, 2021).

\bibitem{Brown17Safe}
N.~Brown, T.~Sandholm, Safe and nested subgame solving for
  imperfect-information games.
\newblock {\it Proceedings of the 31st Conference on Neural Information
  Processing Systems\/} ({NIPS}, 2017).

\bibitem{sustr2020sound}
M.~\v{S}ustr, M.~Schmid, M.~Morav\v{c}\'{i}k, N.~Burch, M.~Lanctot, M.~Bowling,
  Sound search in imperfect information games.
\newblock {\it Proceedings of the International Conference on Autonomous Agents
  and Multiagent Systems\/} ({AAMAS}, 2020).

\bibitem{KovarikFOSG}
V.~Kovar{\'{\i}}k, M.~Schmid, N.~Burch, M.~Bowling, V.~Lis{\'{y}}, Rethinking
  formal models of partially observable multiagent decision making.
\newblock {\it Artificial Intelligence\/} {\bf 303}, 103645 (2022).

\bibitem{Schmid21Thesis}
M.~Schmid, Search in imperfect information games, Ph.D. thesis, Charles
  University (2021). \url{https://arxiv.org/abs/2111.05884}.

\bibitem{KnuthMoore75}
D.~E. Knuth, R.~W. Moore, An analysis of alpha-beta pruning.
\newblock {\it Artificial Intelligence\/} {\bf 6}, 293--326 (1975).

\bibitem{Marsland81}
T.~A. Marsland, M.~Campbell, A survey of enhancements to the alpha-beta
  algorithm.
\newblock {\it ACM Annual Conference\/}, ACM '81 (Association for Computing
  Machinery, New York, NY, USA, 1981), p. 109–114.

\bibitem{Schaeffer1996NewAI}
J.~Schaeffer, A.~Plaat, New advances in alpha-beta searching.
\newblock {\it ACM Conference on Computer Science\/} (Association for Computing
  Machinery, 1996).

\bibitem{Gelly12Go}
S.~Gelly, L.~Kocsis, M.~Schoenauer, M.~Sebag, D.~Silver, C.~Szepesv\'{a}ri,
  O.~Teytaud, The grand challenge of computer {G}o: {M}onte {C}arlo tree search
  and extensions.
\newblock {\it Communications of the ACM\/} {\bf 55}, 106–113 (2012).

\bibitem{Kocsis06UCT}
L.~Kocsis, C.~Szepesvári, Bandit based {M}onte-{C}arlo planning.
\newblock {\it In: ECML-06. Number 4212 in LNCS\/} (Springer, 2006), pp.
  282--293.

\bibitem{Coulom06MCTS}
R.~Coulom, Efficient selectivity and backup operators in {M}onte-{C}arlo tree
  search.
\newblock {\it Computers and Games\/}, H.~J. van~den Herik, P.~Ciancarini,
  H.~H. L. M.~J. Donkers, eds. (Springer Berlin Heidelberg, Berlin, Heidelberg,
  2007), pp. 72--83.

\bibitem{Browne12}
C.~B. Browne, E.~Powley, D.~Whitehouse, S.~M. Lucas, P.~I. Cowling,
  P.~Rohlfshagen, S.~Tavener, D.~Perez, S.~Samothrakis, S.~Colton, A survey of
  monte carlo tree search methods.
\newblock {\it IEEE Transactions on Computational Intelligence and AI in
  Games\/} {\bf 4}, 1-43 (2012).

\bibitem{Silver2017AGZ}
D.~Silver, J.~Schrittwieser, K.~Simonyan, I.~Antonoglou, A.~Huang, A.~Guez,
  T.~Hubert, L.~Baker, M.~Lai, A.~Bolton, Y.~Chen, T.~Lillicrap, F.~Hui,
  L.~Sifre, G.~van~den Driessche, T.~Graepel, D.~Hassabis, {Mastering the game
  of Go without human knowledge}.
\newblock {\it Nature\/} {\bf 550}, 354--359 (2017).

\bibitem{08nips-cfr}
M.~Zinkevich, M.~Johanson, M.~Bowling, C.~Piccione, Regret minimization in
  games with incomplete information.
\newblock {\it Advances in Neural Information Processing Systems 20\/} ({NIPS},
  2008), pp. 905--912.

\bibitem{Hart00}
S.~Hart, A.~Mas-Colell, A simple adaptive procedure leading to correlated
  equilibrium.
\newblock {\it Econometrica\/} {\bf 68}, 1127--1150 (2000).

\bibitem{Tammelin15CFRPlus}
O.~Tammelin, N.~Burch, M.~Johanson, M.~Bowling, Solving heads-up limit {T}exas
  {H}old'em.
\newblock {\it Proceedings of the 24th International Joint Conference on
  Artificial Intelligence\/} (IJCAI, 2015).

\bibitem{Bowling15Poker}
M.~Bowling, N.~Burch, M.~Johanson, O.~Tammelin, Heads-up limit {H}old'em poker
  is solved.
\newblock {\it Science\/} {\bf 347}, 145--149 (2015).

\bibitem{12aamas-pcs}
M.~Johanson, N.~Bard, M.~Lanctot, R.~Gibson, M.~Bowling, Efficient nash
  equilibrium approximation through {M}onte {C}arlo counterfactual regret
  minimization.
\newblock {\it Proceedings of the Eleventh International Conference on
  Autonomous Agents and Multi-Agent Systems\/} ({AAMAS}, 2012).

\bibitem{Burch14CFRD}
N.~Burch, M.~Johanson, M.~Bowling, Solving imperfect information games using
  decomposition.
\newblock {\it Proceedings of the Twenty-Eighth AAAI Conference on Artificial
  Intelligence (AAAI)\/} ({AAAI}, 2014).

\bibitem{brown2020combining}
N.~Brown, A.~Bakhtin, A.~Lerer, Q.~Gong, Combining deep reinforcement learning
  and search for imperfect-information games.
\newblock {\it Thirty-fourth Annual Conference on Neural Information Processing
  Systems\/} ({NeurIPS}, 2020). \url{https://arxiv.org/abs/2007.13544}.

\bibitem{Zarick20Supremus}
R.~Zarick, B.~Pellegrino, N.~Brown, C.~Banister, Unlocking the potential of
  deep counterfactual value networks.
\newblock {\it CoRR\/} {\bf abs/2007.10442} (2020).

\bibitem{Brown18}
N.~Brown, T.~Sandholm, B.~Amos, Depth-limited solving for imperfect-information
  games.
\newblock {\it Proceedings of the Thirty-second Conference on Neural
  Information Processing Systems\/} (NeurIPS, 2018).

\bibitem{CowlingISMCTS}
P.~I. Cowling, E.~J. Powley, D.~Whitehouse, Information set {M}onte {C}arlo
  tree search.
\newblock {\it IEEE Transactions on Computational Intelligence and AI in
  Games\/} {\bf 4}, 120--143 (2012).

\bibitem{Long10Understanding}
J.~Long, N.~R. Sturtevant, M.~Buro, T.~Furtak, Understanding the success of
  perfect information {M}onte {C}arlo sampling in game tree search.
\newblock {\it Proceedings of the Twenty-Fourth AAAI Conference on Artificial
  Intelligence\/}, AAAI'10 ({AAAI}, 2010), p. 134–140.

\bibitem{Nijssen12}
J.~Nijssen, M.~Winands, {M}onte-{C}arlo tree search for the hide-and-seek game
  scotland yard.
\newblock {\it IEEE Transactions on Computational Intelligence and AI in
  Games\/} {\bf 4}, 282--294 (2012).

\bibitem{PimThesis}
J.~Nijssen, Monte-carlo tree search for multi-player games, Ph.D. thesis,
  Maastricht University (2013).

\bibitem{Heinrich2015SmoothUCT}
J.~Heinrich, D.~Silver, Smooth {UCT} search in computer poker.
\newblock {\it Proceedings of the 24th International Joint Conference on
  Artificial Intelligence\/} (IJCAI, 2015).

\bibitem{Lisy15Online}
V.~Lis\'{y}, M.~Lanctot, M.~Bowling, Online {M}onte {C}arlo counterfactual
  regret minimization for search in imperfect information games.
\newblock {\it Proceedings of the Fourteenth International Conference on
  Autonomous Agents and Multi-Agent Systems\/} ({AAMAS}, 2015), pp. 27--36.

\bibitem{09nips-mccfr}
M.~Lanctot, K.~Waugh, M.~Zinkevich, M.~Bowling, {M}onte {C}arlo sampling for
  regret minimization in extensive games.
\newblock {\it Advances in Neural Information Processing Systems 22\/} ({NIPS},
  2009), pp. 1078--1086.

\bibitem{Heinrich15FSP}
J.~Heinrich, M.~Lanctot, D.~Silver, Fictitious self-play in extensive-form
  games.
\newblock {\it Proceedings of the 32nd International Conference on Machine
  Learning ({ICML} 2015)\/} (2015).

\bibitem{Lanctot17PSRO}
M.~Lanctot, V.~Zambaldi, A.~Gruslys, A.~Lazaridou, K.~Tuyls, J.~Perolat,
  D.~Silver, T.~Graepel, A unified game-theoretic approach to multiagent
  reinforcement learning.
\newblock {\it Advances in Neural Information Processing Systems\/} (NeurIPS,
  2017).

\bibitem{Li19DNCFR}
H.~Li, K.~Hu, S.~Zhang, Y.~Qi, L.~Song, Double neural counterfactual regret
  minimization.
\newblock {\it Proceedings of the Eighth International Conference on Learning
  Representations\/} ({ICLR}, 2019).

\bibitem{Brown18DeepCFR}
N.~Brown, A.~Lerer, S.~Gross, T.~Sandholm, Deep counterfactual regret
  minimization.
\newblock {\it CoRR\/} {\bf abs/1811.00164} (2018).

\bibitem{steinberger2020dream}
E.~Steinberger, A.~Lerer, N.~Brown, {DREAM}: Deep regret minimization with
  advantage baselines and model-free learning (2020).

\bibitem{Srinivasan18RPG}
S.~Srinivasan, M.~Lanctot, V.~Zambaldi, J.~P\'{e}rolat, K.~Tuyls, R.~Munos,
  M.~Bowling, Actor-critic policy optimization in partially observable
  multiagent environments.
\newblock {\it Advances in Neural Information Processing Systems\/} ({NeurIPS},
  2018).

\bibitem{Lockhart19ED}
E.~Lockhart, M.~Lanctot, J.~P\'{e}rolat, J.-B. Lespiau, D.~Morrill, F.~Timbers,
  K.~Tuyls, Computing approximate equilibria in sequential adversarial games by
  exploitability descent.
\newblock {\it Proceedings of the 28th International Joint Conference on
  Artificial Intelligence\/} (IJCAI, 2019).

\bibitem{hennes2020neural}
D.~Hennes, D.~Morrill, S.~Omidshafiei, R.~Munos, J.~Perolat, M.~Lanctot,
  A.~Gruslys, J.-B. Lespiau, P.~Parmas, E.~Duenez-Guzman, K.~Tuyls, Neural
  replicator dynamics.
\newblock {\it Proceedings of the International Conference on Autonomous Agents
  and Multiagent Systems\/} ({AAMAS}, 2020).

\bibitem{gruslys2020advantage}
A.~Gruslys, M.~Lanctot, R.~Munos, F.~Timbers, M.~Schmid, J.~Perolat,
  D.~Morrill, V.~Zambaldi, J.-B. Lespiau, J.~Schultz, M.~G. Azar, M.~Bowling,
  K.~Tuyls, The advantage regret-matching actor-critic (2020).

\bibitem{perolat2020poincare}
J.~Perolat, R.~Munos, J.-B. Lespiau, S.~Omidshafiei, M.~Rowland, P.~Ortega,
  N.~Burch, T.~Anthony, D.~Balduzzi, B.~D. Vylder, G.~Piliouras, M.~Lanctot,
  K.~Tuyls, From {P}oincar\'e recurrence to convergence in imperfect
  information games: Finding equilibrium via regularization.
\newblock {\it Proceedings of the The Thirty-eighth International Conference on
  Machine Learning (ICML)\/} (2021).

\bibitem{McAleer21XDO}
S.~McAleer, J.~Lanier, P.~Baldi, R.~Fox, Xdo: A double oracle algorithm for
  extensive-form games.
\newblock {\it Proceedings of the Thirty-fifth Conference on Neural Information
  Processing Systems\/} ({NeurIPS}, 2021).

\bibitem{Feng21NAC}
X.~Feng, O.~Slumbers, Z.~Wan, B.~Liu, S.~M. McAleer, Y.~Wen, J.~Wang, Y.~Yang,
  Neural auto-curricula in two-player zero-sum games.
\newblock {\it Proceedings of the Thirty-fifth Conference on Neural Information
  Processing Systems\/} ({NeurIPS}, 2021).

\bibitem{Perolat22Stratego}
J.~Perolat, B.~D. Vylder, D.~Hennes, E.~Tarassov, F.~Strub, V.~de~Boer,
  P.~Muller, J.~T. Connor, N.~Burch, T.~Anthony, S.~McAleer, R.~Elie, S.~H.
  Cen, Z.~Wang, A.~Gruslys, A.~Malysheva, M.~Khan, S.~Ozair, F.~Timbers,
  T.~Pohlen, T.~Eccles, M.~Rowland, M.~Lanctot, J.-B. Lespiau, B.~Piot,
  S.~Omidshafiei, E.~Lockhart, L.~Sifre, N.~Beauguerlange, R.~Munos, D.~Silver,
  S.~Singh, D.~Hassabis, K.~Tuyls, Mastering the game of {Stratego} with
  model-free multiagent reinforcement learning.
\newblock {\it Science\/} {\bf 378}, 990-996 (2022).

\bibitem{Dash2018SY}
T.~Dash, S.~N. Dambekodi, P.~N. Reddy, A.~Abraham, {Adversarial neural networks
  for playing hide-and-search board game Scotland Yard}.
\newblock {\it Neural Computing and Applications\/} {\bf 32}, 3149--3164
  (2018).

\bibitem{foerster2019bayesian}
J.~N. Foerster, F.~Song, E.~Hughes, N.~Burch, I.~Dunning, S.~Whiteson,
  M.~Botvinick, M.~Bowling, Bayesian action decoder for deep multi-agent
  reinforcement learning (2019).

\bibitem{sokota2021solving}
S.~Sokota, E.~Lockhart, F.~Timbers, E.~Davoodi, R.~D'Orazio, N.~Burch,
  M.~Schmid, M.~Bowling, M.~Lanctot, Solving common-payoff games with
  approximate policy iteration.
\newblock {\it Proceedings of the Thirty-Fifth AAAI Conference on Artificial
  Intelligence\/} ({AAAI}, 2021).

\bibitem{strouse2021collaborating}
D.~Strouse, K.~R. McKee, M.~Botvinick, E.~Hughes, R.~Everett, Collaborating
  with humans without human data.
\newblock {\it Proceedings of the Thirty-fifth Conference on Neural Information
  Processing Systems\/} (NeurIPS, 2021).

\bibitem{Bahkin22Cicero}
A.~Bakhtin, N.~Brown, E.~Dinan, G.~Farina, C.~Flaherty, D.~Fried, A.~Goff,
  J.~Gray, H.~Hu, A.~P. Jacob, M.~Komeili, K.~Konath, M.~Kwon, A.~Lerer,
  M.~Lewis, A.~H. Miller, S.~Mitts, A.~Renduchintala, S.~Roller, D.~Rowe,
  W.~Shi, J.~Spisak, A.~Wei, D.~Wu, H.~Zhang, M.~Zijlstra, Human-level play in
  the game of {Diplomacy} by combining language models with strategic
  reasoning.
\newblock {\it Science\/} {\bf 378}, 1067-1074 (2022).

\bibitem{Stockfish}
T.~S.~D. Team, Stockfish: Open source chess engine (2021).
  \url{https://stockfishchess.org/}.

\bibitem{johanson2013measuring}
M.~Johanson, Measuring the size of large no-limit poker games (2013).

\bibitem{syspiel}
Game of the year 1983: Scotland {Y}ard,
  \url{https://www.spiel-des-jahres.de/spiel-des-jahres-1983-scotland-yard/}.

\bibitem{Southey05bayes}
F.~Southey, M.~Bowling, B.~Larson, C.~Piccione, N.~Burch, D.~Billings,
  C.~Rayner, Bayes' bluff: {O}pponent modelling in poker.
\newblock {\it Proceedings of the Twenty-First Conference on Uncertaintyin
  Artificial Intelligence (UAI)\/} (2005), pp. 550--558.

\bibitem{Elo1986rating}
A.~Elo, {\it The Rating of Chessplayers, Past and Present\/} (Arco Pub., 1986),
  second edn.

\bibitem{gnugo}
T.~G.~D. Team, Gnugo (2009). \url{https://www.gnu.org/software/gnugo/}.

\bibitem{pachi}
P.~Baudis, J.~loup Gailly, Lemonsqueeze, Pachi: Software for the board game of
  go / weiqi / baduk (2016). \url{https://pachi.or.cz/}.

\bibitem{Jackson13SlumbotNL}
E.~Jackson, Slumbot {NL}: Solving large games with counterfactual regret
  minimization using sampling and distributed processing.
\newblock {\it Proceedings of the Computer Poker and Imperfect Information:
  Papers from the AAAI 2013 Workshop\/} (2013).
  \url{https://github.com/ericgjackson/slumbot2019}.

\bibitem{JacksonSlumbotGithub}
E.~Jackson, Slumbot github repository.
  \url{https://github.com/ericgjackson/slumbot2017}.

\bibitem{burch2017aivat}
N.~Burch, M.~Schmid, M.~Moravčík, M.~Bowling, {AIVAT}: A new variance
  reduction technique for agent evaluation in imperfect information games
  (2017).

\bibitem{lisy2017eqilibrium}
V.~Lis\'{y}, M.~Bowling, Eqilibrium approximation quality of current no-limit
  poker bots.
\newblock {\it Workshops at the Thirty-First AAAI Conference on Artificial
  Intelligence\/} (2017).

\bibitem{SegalVirtualLoss}
R.~B. Segal, On the scalability of parallel {UCT}.
\newblock {\it {CG}'10: Proceedings of the 7th international conference on
  Computers and games\/} (2010), pp. 36--47.

\bibitem{huber64}
P.~J. Huber, {Robust Estimation of a Location Parameter}.
\newblock {\it The Annals of Mathematical Statistics\/} {\bf 35}, 73 -- 101
  (1964).

\bibitem{brown2016strategy}
N.~Brown, T.~Sandholm, Strategy-based warm starting for regret minimization in
  games.
\newblock {\it Proceedings of the Thirtieth {AAAI} Conference on Artificial
  Intelligence\/} (AAAI, 2016), pp. 432--438.

\bibitem{LanctotEtAl2019OpenSpiel}
M.~Lanctot, E.~Lockhart, J.-B. Lespiau, V.~Zambaldi, S.~Upadhyay,
  J.~P\'{e}rolat, S.~Srinivasan, F.~Timbers, K.~Tuyls, S.~Omidshafiei,
  D.~Hennes, D.~Morrill, P.~Muller, T.~Ewalds, R.~Faulkner, J.~Kram\'{a}r,
  B.~D. Vylder, B.~Saeta, J.~Bradbury, D.~Ding, S.~Borgeaud, M.~Lai,
  J.~Schrittwieser, T.~Anthony, E.~Hughes, I.~Danihelka, J.~Ryan-Davis,
  {OpenSpiel}: A framework for reinforcement learning in games.
\newblock {\it CoRR\/} {\bf abs/1908.09453} (2019).

\bibitem{sustr2019montecarlo}
M.~{\v{S}}ustr, V.~Kova{\v{r}}{\'\i}k, V.~Lis{\`y}, {M}onte {C}arlo continual
  resolving for online strategy computation in imperfect information games.
\newblock {\it Proceedings of the 18th International Conference on Autonomous
  Agents and MultiAgent Systems\/} (2019), pp. 224--232.

\end{thebibliography}
\bibliographystyle{ScienceAdvances}

\fi

\section*{Acknowledgements} 
%
%

We thank several people for their help and feedback: Ed Lockhart, Michael Johanson, Adam White, Julian Schrittwieser, Thomas Hubert, Michal Sustr, Leslie Acker, Morgan Redshaw, Dustin Morrill, Trevor Davis, Stephen McAleer, Sanah Choudry, and Shayna Bowling.

We would like to extend a special thanks to Mark Winands and Pim Nijssen for providing the code for and helping us configure their Scotland Yard agent, and special thanks to Eric Jackson for providing the code for and assistance with his poker agent.

\noindent \textbf{Author Contributions:}
Conceptualization: NBa, MB, NBu, MM, MS, KW;
Methodology: NBu, MM, MS, KW;
Software: NBa, NBu, JD, ZH, RK, ML, MM, MS, FT, KW;
Validation: NBa, JD, RK, KW;
Formal analysis: NBu;
Investigation: NBa, NBu, JD, ED, RK, ML, MM, MS, FT, KW;
Writing --- original draft: NBa, MB, NBu, RK, ML, MM, MS;
Writing --- review \& editing: NBa, MB, NBu, ML;
Visualization: NBa, JD, RK;
Supervision: NBa, MB, MS;
Project administration: NBa, MB, AC, JD.

\noindent \textbf{Competing Interests:} The authors declare they have no competing interest.

\noindent \textbf{Data and Materials Availability:} All data needed to evaluate the conclusions of the paper are present in the paper and/or the Supplementary Materials.


\ifcamera
\else

\clearpage
\begin{figure}[t!]
\centering
  \ifcamera
  \else
    \includegraphics[width=0.4\textwidth]{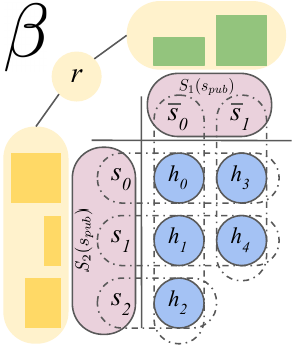}
  \fi
  
  \captionof{figure}{\textbf{An example structure of public belief state $\beta = (\spub, r)$.} $\spub$ translates to two sets of information states, one for player 1, $\cS_1(\spub) = \{\bar{s}_0, \bar{s}_1\}$, and one for player 2, $\cS_2(\spub) = \{s_0, s_1, s_2\}$. Each information state includes different partitions of possible histories. Finally $r$ contains reach probabilities for information states for both players.}
  \label{fig:public-belief-state}
\end{figure}

\clearpage
\begin{figure}[t!]
\centering
  \ifcamera
  \else
    \includegraphics[width=0.9\textwidth]{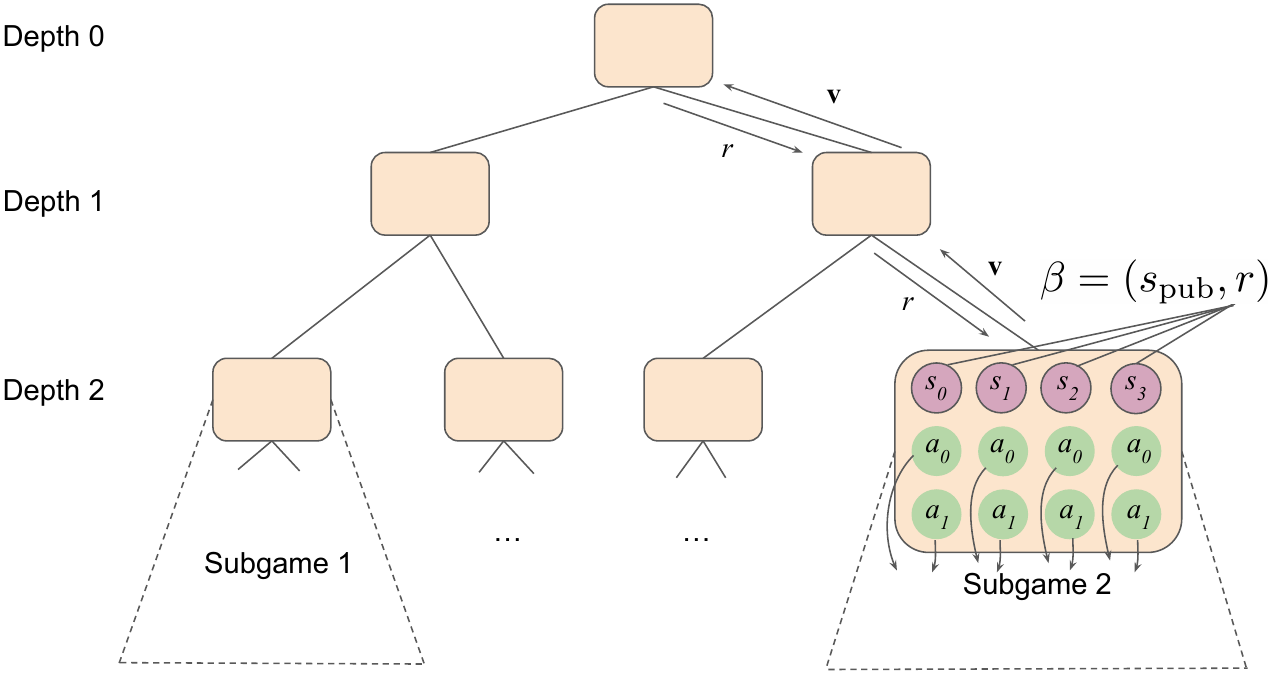}
  \fi
  \caption{\textbf{An example of depth-limited CFR solving using decomposition in a game with two specific subgames shown.} Standard CFR would require traversing all the subgames. Depth-limited CFR decomposes the solve into running down to depth $d=2$ and using $\bv = \bv_\btheta(\beta)$ to represent the second subgame's values. On the downward pass, ranges $r$ are formed from policy reach probabilities. Values are passed back up to tabulate accumulating regrets. Re-solving a subgame would require construction of an auxiliary game~\cite{Burch14CFRD} (not shown).}
\label{fig:decomp}
\end{figure}

\clearpage
\begin{figure}[t!]
    \centering

    \ifcamera
    \else
      \includegraphics[width=\textwidth]{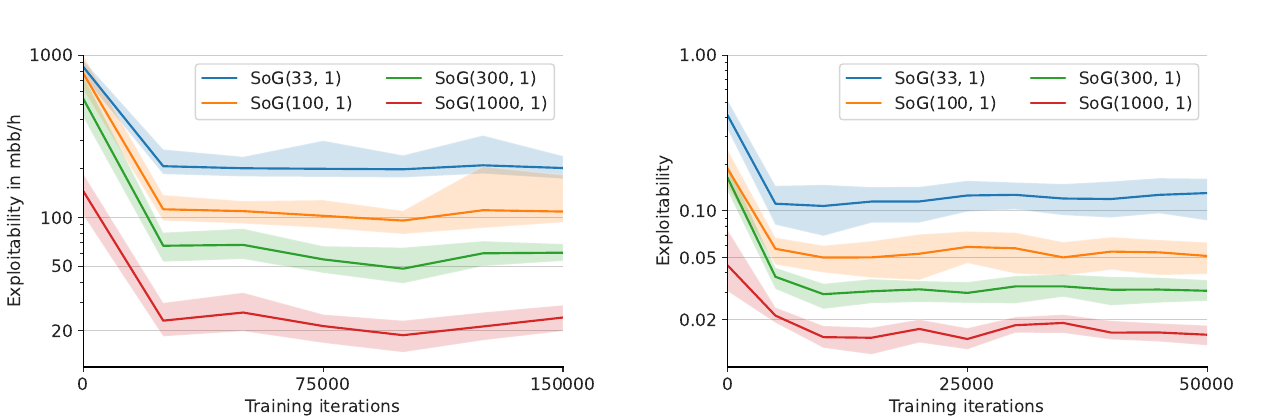}

     \begin{subfigure}[b]{0.48\textwidth}
         \centering
         \caption{Leduc poker}
     \end{subfigure}
     \hfill
     \begin{subfigure}[b]{0.48\textwidth}
         \centering
         \caption{Scotland Yard (glasses map)}
     \end{subfigure}
    \fi
 
    \caption{\textbf{Exploitability of \pog as a function of the number of training steps under different number of simulations of GT-CFR.} For both \textsf{(A)} Leduc poker and \textsf{(B)} Scotland Yard (glasses map), 
    each line corresponds to a different evaluation condition, \eg $\pog(s,c)$ used at evaluation time. 
    The ribbon shows minimum and maximum exploitability out of 50 seeded runs for each setup. The units of the y-axis in Leduc poker are milli big blinds per hand (mbb/h), which corresponds to one thousandth of a chip in Leduc. In Scotland Yard the reward is either -1 (loss) or +1 (win).
    All networks were trained using a single training run of $\pog(100,1)$, and the x-values correspond to a network trained for the corresponding number of steps.
       }
    \label{fig:exploitability}
\end{figure}

\clearpage
\begin{table}[t]
\centering
\begin{tabular}{lrclr}
\toprule
                    Chess Agents &  Rel. Elo && Go Agents &  Rel. Elo \\
\cmidrule(lr){1-2}\cmidrule(lr){4-5}
            AlphaZero(sims=60k) &  +592       && AlphaZero(s=16k, t=800k) & +3139 \\
 Stockfish(threads=16, time=4s) &  +530       && AlphaZero(s=8k, t=800k) & +2875 \\
             AlphaZero(sims=8k) &  +455       && \textbf{SoG(s=16k, c=10)} & \textbf{+1970}\\
          \bf{SoG(s=60k, c=10)} &  \bf{+420}  && \textbf{SoG(s=8k, c=10)} & \textbf{+1902}\\
  Stockfish(threads=4, time=1s) &  +382       && Pachi(s=100k) & +869 \\
           \bf{SoG(s=8k, c=10)} &  \bf{+268}  && Pachi(s=10k) & +231\\
Stockfish(threads=1, time=0.1s) &    0        && GnuGo(l=10) & 0 \\
\bottomrule
\end{tabular}
\caption{\textbf{Relative Elo of different agents in chess (left), and Go (right).} Each agent played 200 matches (100 as white and 100 as black) against every other agent in the tournament.  For chess, Elo of Stockfish with a single thread and 100ms thinking time was set to be 0.  For Go, 
Elo of GnuGo was set to be 0.  The other values are relative to those.  AlphaZero(s=16k, t=800k) refers to 16000 search simulations. For full results, see Tables~\ref{tab:full-go-results} and \ref{tab:full-go-results-recursive}.}
\label{tab:elo-results}
\end{table}

\clearpage
\begin{figure}
  \centering

  \ifcamera
  \else
    \includegraphics[width=\textwidth]{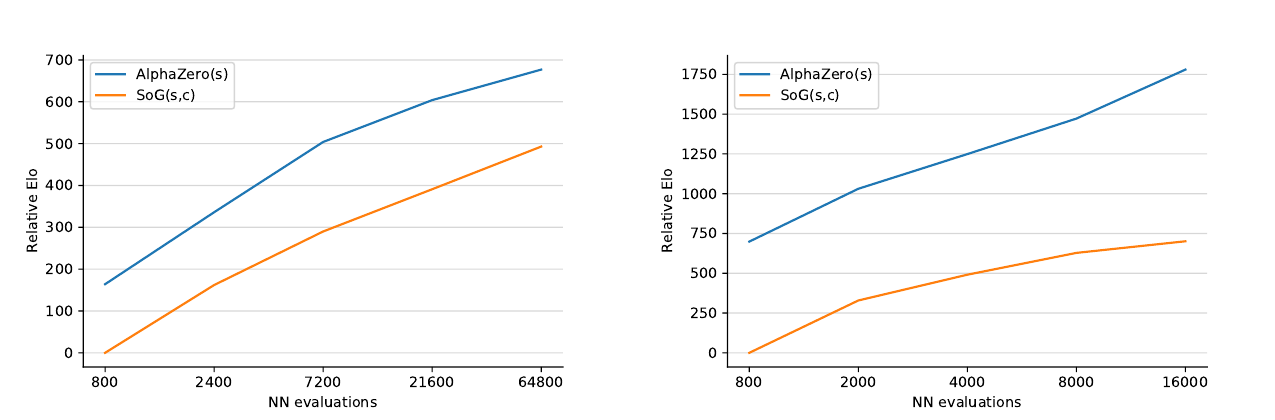}

    \begin{subfigure}[b]{0.48\textwidth}
      \centering
      \caption{Chess}
    \end{subfigure}
    \hfill
    \begin{subfigure}[b]{0.48\textwidth}
      \centering
      \caption{Go}
    \end{subfigure}
    \fi
    \caption{\textbf{Scalability of \pog with increasing number of neural network evaluations compared to AlphaZero measured on relative Elo scale.} The x-axis corresponds to the number of simulations in AlphaZero and $s$ in $\pog(s,c)$.  Elo of $\pog(s=800, c)$ was set to be 0. In chess \textsf{(A)}, $c = 10$ for all runs, with varying $s \in \{ 800, 2400, 7200, 21600, 64800 \}$. In Go \textsf{(B)}, we graph \pog using $(s,c) \in \{ (800, 1), (2000, 10), (4000, 10), (8000, 10), (16000, 16)$.\}}
        \label{fig:pog-and-az-scaling}
\end{figure}


\clearpage
\begin{figure}[t]
\begin{minipage}{.45\textwidth}
  \centering
\begin{tabular}{ l | c c  }
\toprule
{\bf Agent Name} & {\bf Slumbot} & {\bf LBR}~\cite{lisy2017eqilibrium}\\ 
\midrule
Slumbot (2016) & - & -522 $\pm$ 50 \\ 
ARMAC~\cite{gruslys2020advantage}  & - & -460 $\pm$ 260 \\ 
DeepStack~\cite{Moravcik17DeepStack}  & - & $428 \pm 87$\\ 
Modicum~\cite{Brown18}  & 11 $\pm$ 5  & -\\ 
ReBeL~\cite{brown2020combining}  & 45 $\pm$ 5 & - \\ 
Supremus~\cite{Zarick20Supremus} & 176 $\pm$ 44 & 951 $\pm$ 96 \\ 
\midrule
$\pog(10, 0.01)$  &  7 $\pm$ 3 & 434 $\pm$ 9\\
\bottomrule
\end{tabular}
\vspace{-0.05in}
\captionof{table}{\small{\textbf{Head-to-head results showing expected winnings (mbb/h) of \pog and other recently published agents against Slumbot and LBR.} The LBR agent use either fold or call (FC) actions in the all four rounds. The $\pm$ shows one standard error. LBR results for Slumbot are from~\cite{lisy2017eqilibrium}. The other results are from the papers describing the agents.
}}
\label{tab:hunl}
\end{minipage}
\hspace{0.05\textwidth}
\begin{minipage}{.45\textwidth}
  \centering
  \ifcamera
  \else
    \vspace{-0.5cm}\includegraphics[width=\textwidth]{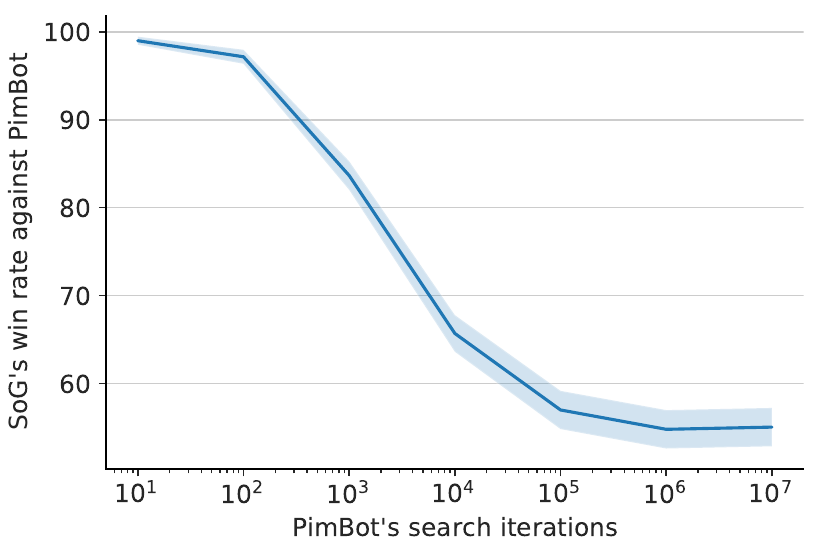}
  \fi
  \captionof{figure}{\textbf{Win rate of $\pog(400,1)$ against PimBot with varying simulations.} 2000 matches were played for each data point, with roles swapped for half of the matches. Note that the x-axis has logarithmic scale. The ribbon shows 95\% confidence interval.}
  \label{fig:pog_x_pimbot}
\end{minipage}
\end{figure}

\clearpage
\begin{figure}[t!]
  \centering
  
  \ifcamera
  \else
    \includegraphics[width=0.9\textwidth]{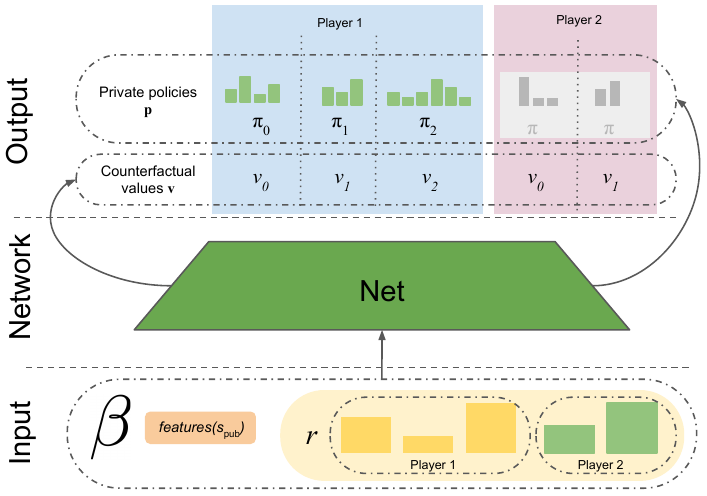}
  \fi
  \caption{\textbf{A counterfactual value-and-policy network (CVPN).}
    Each query, $\beta$, to the network includes beliefs $r$ and an encoding of $\spub$ to get the counterfactual values $\bv$ for both players and policies $\bp$ for the acting player in each information state $s_i \in \spub(h)$, producing outputs $f_\btheta$. Since players may have different actions spaces (as in \eg Scotland Yard) there are two sets of policy outputs: one for each player, and $\bp$ refers to the one for the acting player at $\spub$ only (depicted as player 1 in this diagram by greying out player 2's policy output).}
  \label{fig:cvpn}
\end{figure}

\clearpage
\begin{figure}[t!]
  \centering
  \ifcamera
  \else
    \includegraphics[width=0.9\textwidth]{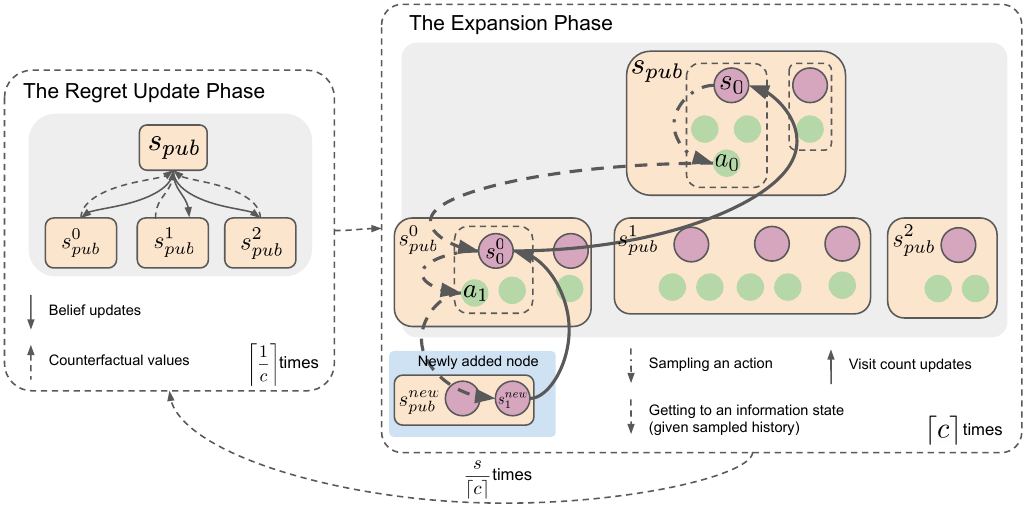}
  \fi
  \caption{\textbf{Overview of the phases in one iteration of Growing-Tree CFR.} The regret update phase propagates beliefs down the tree, obtains counterfactual values from the CPVN at leaf nodes (or from the environment at terminals), and passes back counterfactual values to apply the CFR update. The expansion phase simulates a trajectory from the root to a leaf, adding public states to the tree. In this case the trajectory starts in the public belief state $s_{pub}$ by sampling the information state $s_0$. After that the sampled action $a_0$ leads to the information state $s^0_0$ in public state $s^0_{pub}$, and finally the action $a_1$ leads to a new public state that is added to the tree.}
  \label{fig:pog-search}
\end{figure}

\clearpage
\begin{figure}[t!]
  \centering
  
  \ifcamera
  \else
    \includegraphics[width=0.9\textwidth]{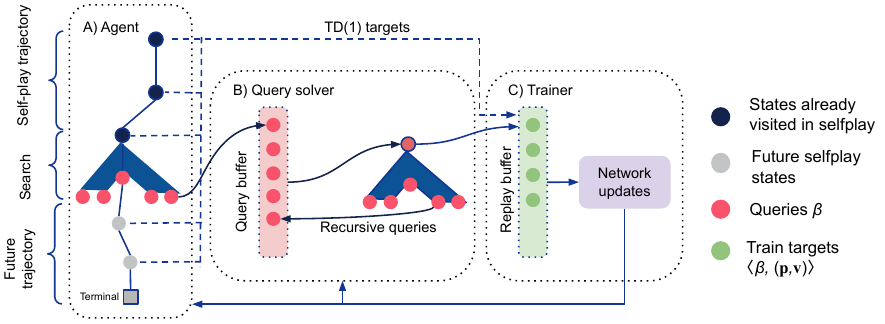}
  \fi
  
  \caption{\textbf{\pog Training Process.}
  Actors collect data via sound self-play and trainers run separately over a distributed network.
    \textsf{(A)} Each search produces a number of CVPN \defword{queries} with input $\beta$. 
    \textsf{(B)} Queries are added to a query buffer and subsequently solved by a \defword{solver} that studies the situation more closely via another invocation of GT-CFR. During solving, new recursive queries might be added back to the query buffer; separately the network is 
    \textsf{(C)} trained on minibatches sampled from the replay buffer to predict values and policy targets computed by the solver.}
  \label{fig:pog-training}
\end{figure}

\clearpage
\setcounter{table}{0} 
\setcounter{figure}{0} 
\renewcommand{\thetable}{S\arabic{table}}  
\renewcommand{\thefigure}{S\arabic{figure}}

\begin{center}
  \makebox{\includegraphics[width=0.7\paperwidth]{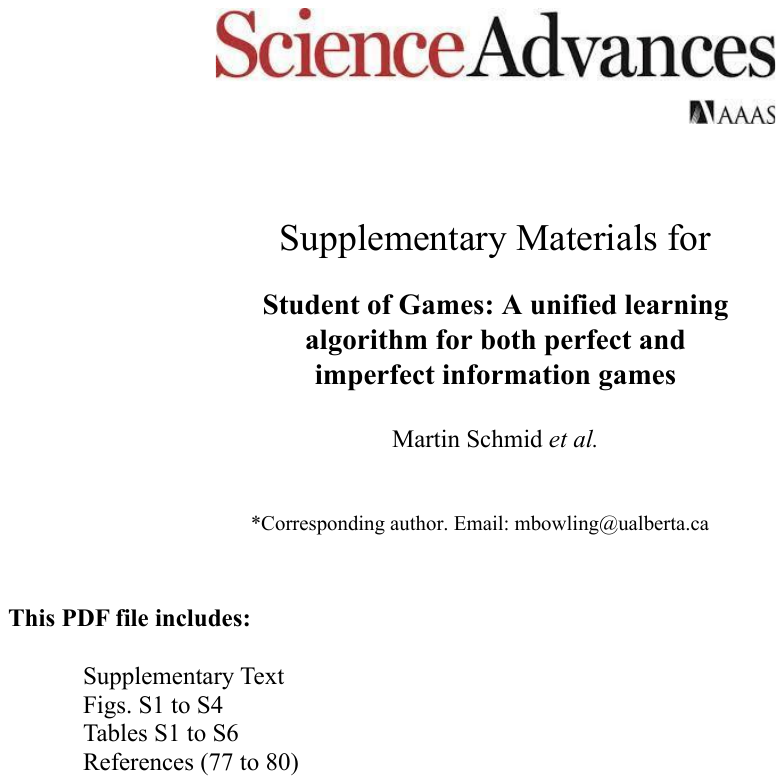}}
\end{center}

\section*{Supplementary Text}

\setul{1pt}{.6pt}
\titleformat{\section}{\Large}{\thesection}{1em}{\ul}
\titleformat{\subsection}{\large}{\thesubsection}{1em}{\ul}
\titleformat{\subsubsection}{\normalsize}{\thesubsubsection}{1em}{\ul}

\section*{Student of Games Algorithm Details}
\label{sec:alg-details}

\subsection*{Network Architecture and Optimization}
\label{sec:net-arch}

Table~\ref{tab:nn-architectures-per-game} lists neural network architectures and input features used for each game. For chess and Go we use exactly the same architecture and inputs as used by AlphaZero~\cite{Silver18AlphaZero}. In poker and Scotland Yard we process concatenated belief and public state features by a MLP with ReLU activations. 

The counterfactual value head is optimized by Huber loss~\cite{huber64}, while policy for each information state $i$ is optimized by KL divergence:

$$l(\bv,\bp,\bv_{target},\bp_{target}) = w_v * l_{huber}(\bv,\bv_{target}) + w_p * \sum_{i} {l_{KL}(\pi^i, \pi^i_{target})}$$ 

where each head is weighted with the corresponding weight $w_v$ and $w_p$. During training we smoothly decay the learning rate by a factor of $d$ every $T_{decay}$ steps. Formally learning rate $\alpha$ at training step $t$ is defined as: 

$$\alpha_t = \alpha_{init} * d^{t/T_{decay}}$$

When using the policy head's prediction as prior in PUCT formula the logits are processed with softmax with temperature $T_{prior}$. This can decrease weight of the prior in some games and encourage more exploration in the search phase.

\subsection*{Pseudocode}
\label{sec:pseudocode}

Here we provide pseudocode for the most important parts of the \pog algorithm. Algorithm~\ref{alg:gt-cfr} specifies GT-CFR, the core of \pog's sound game-theoretic search that scales to large perfect information games introduced in \sectionref{sec:alg-gtcfr}. Algorithm~\ref{alg:selfplay} presents how GT-CFR is used during self-play that generates training examples for the neural network, previously covered in \sectionref{sec:training}. Hyperparameters used in self-play are specified in Table~\ref{tab:hparams-per-game}. 

When \pog plays against an opponent, the search tree is rebuilt also for the opponent's actions (as discussed in \sectionref{sec:modified-resolving}).  This way, \pog reasons about the opponent's behavior since it directly influences the belief distribution for the current state $\beta$ where \pog is to act.

\algnewcommand{\LineComment}[1]{\State \(\triangleright\) #1}

\begin{algorithm}[p]
\caption{Growing Tree CFR. Note that GT-CFR is logging all neural net queries it does since they might be used later in training.}
\label{alg:gt-cfr}
\begin{algorithmic}

\Procedure{GT-CFR}{$\cL^0, \beta, s, c$}
  \LineComment{$\cL^0$ --- a tree including $\beta$ built as described in \sectionref{sec:modified-resolving}.}
  \LineComment{$\beta$ --- a public belief state under which the new nodes will be added.}
  \LineComment{$s, c$ --- total number of expansion simulations and number of simulations per CFR update.}
  
  \For{$i \in \{0, 1, \cdots, \frac{s}{\lceil c \rceil} -1 \}$}
    \State \Call{CFR}{$\cL^i, \left \lceil \frac{1}{c} \right \rceil$} \Comment{Store average policy and counterfactual values in the tree.}
    \State $\cL^{i+1} \gets \Call{Grow}{\cL^i}$
  \EndFor
  \LineComment{Return counterfactual values and average policy from CFR and all NN calls.}
  \State \Return{$\bv, \bp, nn\_queries$} 
\EndProcedure

\State

\Procedure{Grow}{$\cL, \beta$}
    \For{$i \in \{ 0, 1, \cdots, \lceil c \rceil - 1\}$}
      \State $path \gets$ \Call{SamplePathDownTheTree}{$\cL, \beta$}    \Comment{The path starts at $\beta$.}
      \State \Call{AddTopKChildren}{$\cL, path, k$}
      \State \Call{UpdateVisitCountsUp}{$\cL, path$}
     \EndFor
    \State \Return{$\cL$}
\EndProcedure

\end{algorithmic}
\end{algorithm}

\begin{algorithm}
\caption{Sound Self-play}
\label{alg:selfplay}
\begin{algorithmic}

\Procedure{SelfPlay}{}
  \State Get initial history state $w \gets w^{INIT}$ and corresponding public state $\beta$ 
  \LineComment Decide whether the game has to be played till the end.
  \State $do\_not\_resign \gets $ coin flip with probability $p_{no\_resign}$
  \While{$w$ is not terminal AND played less than $moves_{max}$}
    \If{chance acts in $w$}
        \State {$a \gets$ uniform random action.}
    \Else
        \LineComment{\pog acts for all non-chance players.}
        \State $v_w, \pi^{controller}_w \gets$ \Call{SoGSelfplayController}{$w$}
        \If{$v_w < resign\_threshold$ AND $not(do\_not\_resign)$}
            \LineComment Don't waste compute on already decided game.
            \State \Return{}
        \EndIf
        \LineComment Mix controller's policy with uniform prior to encourage exploration.
        \State $\pi^{selfplay}_w \gets (1-\epsilon) \cdot \pi^{controller}_w + \epsilon \cdot \pi^{uniform}$ 
        \If{moves played $< moves_{greedy\_after}$}
            \State $a \gets$ sample action from $\pi^{selfplay}_w$
        \Else
            \State $a \gets \arg \max \pi^{selfplay}_w$
        \EndIf
    \EndIf
    \State $w \gets$ apply action $a$ on state $w$
  \EndWhile
  \LineComment{Sampling states with TD(1) targets.} 
  \For{each belief state $\beta \in$ played trajectory $tr$} 
    \If{uniform random sample from unit interval $< p_{td1}$}
      \State $\bv \gets$ outcome of $tr$ assigned to information state visited in $\beta$
      \State $\bp \gets$ policy used in $\beta$
      \State replay\_buffer.append($\langle \beta,(\bv, \bp) \rangle$)
    \EndIf
  \EndFor
\EndProcedure

\algstore{selfplay}
\end{algorithmic}
\end{algorithm}

\begin{algorithm}
\begin{algorithmic}
\algrestore{selfplay}

\Procedure{SoGSelfplayController}{$w$}
    \State $\beta \gets$ public state including $w$
    \State $\cL \gets$ the tree including $\beta$ built as described in \sectionref{sec:modified-resolving}. 
    \State $\bv, \bp \gets$ \Call{Training-GT-CFR}{$\cL$}
    \State \Return $\bv(w), \bp(w)$
\EndProcedure

\State

\Procedure{Training-GT-CFR}{$\cL$}
  \State $\bv, \bp, nn\_queries \gets$ \Call{GT-CFR}{$\cL$}
  \State queries $\gets$ Pick on average $q_{search}$ neural net queries $\beta$ from $nn\_queries$.
  \State queries\_to\_solve.extend(queries)
  \State \Return $\bv, \bp$
\EndProcedure

\State

\Procedure{QuerySolver}{}
  \For{$\beta \gets$ queries\_to\_solve.pop()}
    \State $\bv, \bp, nn\_queries \gets$ \Call{GT-CFR}{$\beta$}
    \LineComment{Send the example to the trainer.}
    \State replay\_buffer.append($\langle \beta,(\bv, \bp) \rangle$)
    \LineComment{Create recursive queries.}
    \State queries $\gets$ Pick on average $q_{recursive}$ neural net queries $\beta$ from $nn\_queries$.
  \State queries\_to\_solve.extend(queries)
  \EndFor
\EndProcedure
\end{algorithmic}
\end{algorithm}

Note that unlike AlphaZero, \pog currently starts its search procedure from scratch.
That is, the previous computation only provides invariants for the next resolving step.
AlphaZero rather warm-starts the MCTS process by initializing values and visit counts from the previous search.
For \pog, this would also require warm-starting CFR.
While possible \cite{brown2016strategy}, there is no warm-starting in the current implementation of \pog.

\subsection*{Implementation}
\pog is implemented as a distributed system with decoupled actor and trainer jobs. Each actor runs several parallel games and the neural network evaluations are batched for better accelerator utilization. The networks were implemented using TensorFlow. 

\subsection*{Poker Betting Abstraction}
\label{sec:poker-bets}

There are up to 20000 possible actions in no-limit Texas hold'em. To make the problem easier, AI agents are typically allowed to use only a small subset of these~\cite{Moravcik17DeepStack,Brown18,brown2020combining,Zarick20Supremus}. This process of selecting a set of allowed actions for a given poker state is called betting abstraction. Even using betting abstraction the players are able to maintain strong performance in the full game~\cite{Moravcik17DeepStack,brown2020combining,Zarick20Supremus}. Moreover, the local best response evaluation~\cite{lisy2017eqilibrium} suggests that there is not an easy exploit for such simplification as long as the agent is able to see full opponent actions~\cite{Moravcik17DeepStack}. 

We use a betting abstraction in the \poglong to speed up the training and simplify the learning task. 
Our agent's action set was limited to just 3 actions: fold (give up), check/call (match the current wager) and bet/raise (add chips to the pot). 
To improve generalization we used stochastic betting size similarly to ReBeL~\cite{brown2020combining}. The single allowed bet/raise size is randomly uniformly selected at the start of each poker hand from the interval $\langle0.5, 1.0\rangle * pot\_ size$. This amount is anecdotally similar to one used by human players and had good performance in our experiments. The same random selection was used in both training and evaluation.

As in~\cite{brown2020combining}, we have also randomly varied the stack size (number of chips available to the players) at the start of the each round during the training. This number stays fixed during evaluation.  

\section*{Evaluation Details and Additional Experimental Results}
\label{app:additional-results}

\subsection*{Description of Leduc poker}

Leduc is a simplified poker game with two rounds
and a 6-card deck in two suits. Each player initially antes a
single chip to play and obtains a single private card and
there are three actions: fold, call and raise. There is a fixed
bet amount of 2 chips in the first round and 4 chips in the
second round, and a limit of two raises per round. After
the first round, a single public card is revealed. A pair is
the best hand, otherwise hands are ordered by their high
card (suit is irrelevant). A player's reward is their gain or loss in chips
after the game.

\subsection*{Reinforcement Learning and Search in Imperfect Information Games}
\label{app:basic-rl-mcts-results}
In this section, we provide some experimental results showing that common RL and widely-used search algorithms can produce highly exploitable strategies, even in small imperfect information games where exploitability is computable exactly. In particular, we show how exploitable Information Set Monte Carlo Tree Search is in Leduc poker, as well as three standard RL algorithms (DQN, A2C and tabular Q-learning) in both Kuhn poker and Leduc poker using OpenSpiel~\cite{LanctotEtAl2019OpenSpiel}. Results are presented in milli big blinds per hand (mbb/h), which corresponds to one thousandth of a chip for both games.

\subsubsection*{Information Set Monte Carlo Tree Search}
Information Set Monte Carlo Tree Search (IS-MCTS) is a search method that, at the start of each simulation, first samples a world
state-- consistent with the player's information state-- and uses it for the simulation~\cite{CowlingISMCTS}. Reward and visit count statistics are
aggregated over information states so that players base their decisions only on their information states rather than on private
information inaccessible to them. 

Table~\ref{fig:leduc-ismcts-exploitability} shows the exploitability of a policy obtained by running
separate independent IS-MCTS searches from each information state in the game, over various parameter values.
The lowest exploitability of IS-MCTS we found among this sweep was {\bf 465 mbb/h}.

\subsubsection*{Standard RL algorithms in Imperfect Information Games}
As imperfect information games generally need stochastic policies to achieve an optimal strategy, one might wonder how exploitable standard RL algorithms are in this class of games. To test this, we trained three standard RL agents: DQN, policy gradient (A2C) and tabular Q-learning. We used MLP neural networks in DQN and A2C agents.
Table~\ref{fig:RL_hyper_params} shows the hyper parameters we swept over to train these RL agents.

In Kuhn poker, the best performing A2C agent converges to exploitability of 52 mbb/h, and tabular Q-learning and DQN agents converge to around 250 mbb/h. Similarly, in Leduc poker, the best performing A2C agent converges to exploitability of 78 mbb/h, tabular Q-learning and DQN agents converge to about 1300 mbb/h and 900 mbb/h respectively. Figure~\ref{fig:rl-exploitability} shows the exploitability of RL agents in Kuhn poker and Leduc poker.



\section*{Proofs of Theorems}
\label{sec:proofs}
There are three substantive differences between the \pog algorithm and DeepStack.
First, \pog uses a growing search tree, rather than using a fixed limited-lookahead tree.
Second, the \pog search tree may depend on the observed chance events.
Finally, \pog uses a continuous self-play training loop operating throughout the entire game, rather than the stratified bottom-up training process used by DeepStack.
We address each of these differences below, in turn, after considering how to describe an approximate value function for search in imperfect information games.

\subsection*{Value Functions for Subgames}

Like DeepStack, the \pog algorithm uses a value function, so the quality of its play depends on the quality of the value function.
We will describe a value function in terms of its distance to a strategy with low regret.
We start with some value and regret definitions that are better suited to subgames.

Consider some policy profile $\pi$ which is a tuple containing a strategy for each player, public tree subgame $S$ rooted at public state $\pubstate{}$ with player ranges $B_i[\infostate{i}\in{}\mathcal{S}_i(\pubstate{})]\coloneqq{}P_i(\infostate{i}|\pi)$.
First, note that we can re-write counterfactual value $v$ so that it depends only on $B$ and $\pi$ restricted to $S$, with no further dependence on $\pi$.
Let $\infostate{i}$ be a Player $i$ information state in $\mathcal{S}_i(\pubstate{})$, and $q$ be the opponent of Player $i$, then:
\begin{align*}
\cfv{\infostate{i}}{\mathbf{B},\pi^S} & \coloneqq \sum_{h \in I(\infostate{i})} \sum_{z \sqsupset h} B_q[\infostate{q}(h)] P_c(h) P(z|h,\pi^S) u_i(z) \\
& = \sum_{h \in I(\infostate{i})} \sum_{z \sqsupset h} P_{-i}(h|\pi) P(z|h,\pi) u_i(z) = \cfv{\infostate{i}}{\pi}
\end{align*}

We can write several quantities in terms of the best-response value at information state $\infostate{i}$:
\begin{align*}
\brcfv{\infostate{i}}{B, \pi^S} \coloneqq \max_{\pi^{*}_{i}} \cfv{\infostate{i}}{B, \pi^S \leftarrow \pi^{*}_{i}}
\end{align*}
where $\pi \leftarrow \pi'$ is the policy profile constructed by replacing action probabilities in $\pi$ with those in $\pi'$.
The value function is a substitute for an entire subgame policy profile, so the regret we are interested in is player $i$'s full counterfactual regret~\cite{08nips-cfr} at $\infostate{i}$, which considers all possible strategies within subgame $S$:
\begin{align*}
    R^{\text{full}}_{\infostate{i}}(B, \pi^S) \coloneqq \brcfv{\infostate{i}}{B, \pi^S} - \cfv{\infostate{i}}{B, \pi^S}
\end{align*}

With these definitions in hand, we can now consider the quality of a value function $f$ in terms of a regret bound $\epsilon$ and value error $\xi$.
Recall that $f$ maps ranges $B$ and public state $\pubstate{}$ to approximate counterfactual values $\apxcfv{\infostate{i}}{}$ for each player $i$.

First, we consider versions of the regret bound and value error which are parameterised by a strategy $\pi$.
There is some associated bound $\epsilon(\pi)$ on the sum of regrets across all information states at any subgame, valid for both players.
\begin{align*}
    \epsilon(\pi) \coloneqq \max_{B} \max_\pubstate{} \max_{i} \sum_{\infostate{i} \in \mathcal{S}_i(\pubstate{})} R^{\text{full}}_{\infostate{i}}(B,\pi)
\end{align*}
There is also some bound $\xi_f(\pi)$ on the distance between $f(\pubstate{},B)$ and the best-response values to $\pi$.
\begin{align*}
    \xi_{f}(\pi) \coloneqq \max_{B} \max_\pubstate{} \max_{i} \sum_{\infostate{i} \in \mathcal{S}_i(\pubstate{})} |f(\pubstate{},B)[\infostate{i}] - \cfv{\infostate{i}}{B,\pi}|
\end{align*}

We then say that $f$ has $\epsilon,\xi$ quality bounds if there exists some strategy $\pi$ such that $\epsilon(\pi)\le\epsilon$ and $\xi_f(\pi)\le\xi$.
As desired, if both $\epsilon$ and $\xi$ are low then $f(\pubstate{},B)$ is a good approximation of the best-response values to a low-regret strategy, for a subgame rooted at $\pubstate{}$ with initial beliefs $B$.

The DeepStack algorithm~\cite{Moravcik17DeepStack} used a similar error metric for value functions, but only considered zero-regret strategies.
We introduce a more complicated error measure because the space of values corresponding to low-regret strategies may be much larger than the space of values corresponding to no-regret strategies.
For example, consider the public subgame of a matching pennies game after the first player acts with the policy 0.501 heads, 0.499 tails.
There are two first-player information states, from playing either heads or tails, with an empty first-player strategy as there are no further first player actions.
Let us assume a value function $f$ is returning the values $[0\ 0]$ for these two information states.
How good is $f$, assuming we restrict our attention to this one subgame?

The unique zero-regret strategy for the second player is to play tails 100\% of the time, resulting in first player counterfactual values of -1 for playing heads and 1 for playing tails.
The error metric based on zero-regret strategies is therefore measuring $|f([0.501\ 0.499]) - [-1\ 1]|_1$, so that the DeepStack metric states that $f$ has an error of 2.
However, $[0\ 0]$ seems like a very reasonable choice: these are exactly the first player counterfactual values when the second player has a strategy of 0.5 heads, 0.5 tails, which has a regret of only 0.002 in this subgame.
Rather than saying $f$ is a poor quality value function with an error of $2$ in a game with utilities in $[-1,1]$, we can now say $f$ is a great $0.002,0$ value function which exactly describes a low-regret strategy.

The new quality metric also addresses an issue the old DeepStack metric had with discontinuities in the underlying 0-regret value functions.
This means that the space of functions with a low DeepStack error metric may not be well suited for learning from data.
Continuing with the previous example, if we shifted $B$ slightly to be 0.499 heads and 0.501 tails for the first player, the unique 0-regret strategy in the subgame flips to playing tails 0\% of the time, while the uniform random strategy is still a low-regret strategy for this subgame.
In this example, a function can only have a low error with the DeepStack metric if it accurately predicts the values everywhere around the discontinuity at 0.5 heads and 0.5 tails, whereas the new metric can avoid this discontinuity by picking an $\epsilon>0$.
More generally, for any constant $c$, the objective $\epsilon+c\xi$ is a continuous function in $B$, making it a potentially more attractive learning target than the discontinuous function defined by exact Nash equilibrium values, and which matches a learning procedure based on approximately solving example subgames.

\subsection*{Growing Trees}

One major step in showing soundness of the \pog algorithm is demonstrating that Growing Tree CFR (\gtcfr{}) can approximately solve games.
As a quick recap, \gtcfr{} is a variant of the CFR algorithm~\cite{08nips-cfr} that uses limited lookahead and a value function, storing values within a tree that grows over time, in a fashion similar to UCT~\cite{Kocsis06UCT}.
We use this algorithm as a component to solve the problems that the \pog algorithm sets up.
At every non-terminal public leaf state $\pubstate{}$ of the lookahead tree, \gtcfr{} uses estimated counterfactual values $~\tilde{v}$, generated from a value function $f(\pubstate{},B)$ with player ranges $B$ induced by Bayes' rule at $\pubstate{}$ for the current policy profile $\pi$.

Like DeepStack, \pog has two steps which involve solving subgames of the original game.
One of the steps is the re-solving step used to play through a game, where we solve a modified subgame based on constraints on opponent values and beliefs about our possible private information, in order to get our policy and new opponent values.
The other step is only in the training loop, where we are solving a subgame with fixed beliefs for both players, in order to get values for both players.
While the (sub)games for these two cases are slightly different, they are both well-formed games and we can find an approximate Nash equilibrium using \gtcfr{}.

When running \gtcfr{}, even though a policy is explicitly defined only at information states in the lookahead tree $\mathcal{L}$, at each iteration $t$ there is implicitly some complete policy profile $\pi^t$.
For any information state $\infostate{}$ in $\mathcal{L}$ which is not a leaf, $\pi^t(\infostate{})$ is explicitly defined by the regret-matching policy.
For all other $\infostate{}$ -- either a leaf of $\mathcal{L}$ or outside of the lookahead tree -- $\pi^t(\infostate{})$ is defined by the $\epsilon$-regret subgame policy profile $\pi^{*,S}$ associated with the value function's $\epsilon,\xi$ quality bounds.
Note that this $\pi^t$ only exists as a concept which is useful for theoretical analysis: \gtcfr{} does not have access to the probabilities outside of its lookahead tree, only a noisy estimate of the associated counterfactual values provided by the value function.

\def\rbf#1{\rm{\mathbf{#1}}}

\begin{lemma}
\label{lem:cfv_decompose}
Let $\rbf{p}$ and $\rbf{q}$ be vectors in $[0,1]^n$, and $\rbf{v}$ and $\rbf{w}$ be vectors in $\mathbb{R}^n$ such that $\rbf{v}[i]>\rbf{w}[i]$ for all $i$.
Then $\rbf{p} \cdot \rbf{v} - \rbf{q} \cdot \rbf{w} \le \rbf{1} \cdot (\rbf{v} - \rbf{w}) + \rbf{p} \cdot \rbf{w} - \rbf{q} \cdot \rbf{w}$
\end{lemma}
\begin{proof}
\begin{align*}
\rbf{p} \cdot \rbf{v} - \rbf{q} \cdot \rbf{w} & = \rbf{p} \cdot \rbf{v} - \rbf{p} \cdot \rbf{w} + \rbf{p} \cdot \rbf{w} - \rbf{q} \cdot \rbf{w} \\
& = \rbf{p} \cdot (\rbf{v} - \rbf{w}) + \rbf{p} \cdot \rbf{w}  - \rbf{q} \cdot \rbf{w} \\
& \le \rbf{1} \cdot(\rbf{v} - \rbf{w}) + \rbf{p} \cdot \rbf{w} - \rbf{q} \cdot \rbf{w}
\end{align*}
\end{proof}

\begin{lemma}
\label{lem:value_error}
Let $\rbf{p}$ and $\rbf{q}$ be vectors in $[0,1]^n$, and $\rbf{v}$ and $\rbf{w}$ be vectors in $\mathbb{R}^n$ such that $\sum_{i=1}^{n}|\rbf{v}[i]-\rbf{w}[i]|\le\xi$.
Then $(\rbf{p}-\rbf{q})\cdot{}\rbf{v}\le\xi+(\rbf{p}-\rbf{q})\cdot{}\rbf{w}$.
\end{lemma}
\begin{proof}
\begin{align*}
(\rbf{p} - \rbf{q}) \cdot{} \rbf{v} & = (\rbf{p} - \rbf{q}) \cdot (\rbf{v} - \rbf{w}) + (\rbf{p} - \rbf{q}) \cdot \rbf{w} \\
& \le \sum_{i=1}^n |(\rbf{p}[i] - \rbf{q}[i])(\rbf{v}[i] - \rbf{w}[i])| + (\rbf{p} - \rbf{q}) \cdot \rbf{w} \\
& \le \sum_{i=1}^n|(\rbf{v}[i] - \rbf{w}[i])| + (\rbf{p} - \rbf{q}) \cdot \rbf{w} \\
& \le \xi + (\rbf{p} - \rbf{q}) \cdot \rbf{w} \\
\end{align*}
\end{proof}

In \gtcfr{}, the depth-limited public tree used for search may change at each iteration.
Let $\mathcal{L}^t$ be the public tree at time $t$.
For any given tree $\mathcal{L}$, let $\treeinterior{\mathcal{L}}$ be the interior of the tree: all non-leaf, non-terminal public states.
The interior of the tree is where regret matching is used to generate a policy, with regrets stored for all information states in interior public states.
Let $\treefrontier{\mathcal{L}}$ be the frontier of $\mathcal{L}$, containing non-terminal leaves, and $\treeterminals{\mathcal{L}}$ be the terminal public states.
\gtcfr{} uses the value function at all public states in the frontier, receiving noisy estimates $\apxcfv{\infostate{}}{}$ of the true counterfactual values $\cfv{\infostate{}}{}$.
We will distinguish between the true regrets $R^T_\infostate{}$ computed from the entire policy, and the regret $\tilde{R}^T_\infostate{}$ computed using the estimated values $\apxcfv{\infostate{}}{}$.
Given a sequence of trees across $T$ iterations, let $\interiorintervals{\pubstate{}}$ be the set of maximal length intervals $[a,b] \subseteq [1,T]$ where $\pubstate{}$ is in $\treeinterior{\mathcal{L}^t}$ for all $t\in[a,b]$.
Let $U$ be the maximum difference in counterfactual value between any two strategies, at any information state, and $A$ be the maximum number of actions at any information state.
\begin{lemma}
\label{lem:gtcfr_regret}
After running \gtcfr{} for $T$ iterations starting at some initial public state $s_0$, using a value function with quality $\epsilon, \xi$, regret for the strategies satisfies the bound
\begin{align*}
\sum_{\infostate{i} \in \mathcal{S}_i(s_0)} R^T_\infostate{i} & \le \sum_{t=1}^T |\treefrontier{\mathcal{L}^t}|(\epsilon + \xi) \\
& + \sum_{\pubstate \in \bigcup_{t=1}^T \mathcal{L}^t} |\mathcal{S}_i(\pubstate)| U \sqrt{A} \sum_{[a,b] \in \interiorintervals{\pubstate}} \sqrt{|[a,b]|}
\end{align*}
\end{lemma}
\begin{proof}
Starting with the definition of regret, and noting that regrets are independently maximised in a perfect recall game, we can rearrange terms to get
\begin{align*}
\sum_{\infostate{i} \in \mathcal{S}_i(s_0)} R^T_\infostate{i} & = \sum_{\infostate{i} \in \mathcal{S}_i(s_0)} \left( \max_{\pi^*_i} \sum_{t=1}^T \cfv{\infostate{i}}{\pi^t \leftarrow \pi^*_i} - \sum_{t=1}^T \cfv{\infostate{i}}{\pi^t} \right) \\
& = \max_{\pi^*_i} \sum_{\infostate{i} \in \mathcal{S}_i(s_0)} \left( \sum_{t=1}^T \cfv{\infostate{i}}{\pi^t \leftarrow \pi^*_i} - \sum_{t=1}^T \cfv{\infostate{i}}{\pi^t} \right) \\
& = \max_{\pi^*_i} \sum_{t=1}^T \sum_{\infostate{i} \in \mathcal{S}_i(s_0)} \left( \cfv{\infostate{i}}{\pi^t \leftarrow \pi^*_i} - \cfv{\infostate{i}}{\pi^t} \right) \\
\end{align*}
We can rewrite the counterfactual values of information state $\infostate{i}$ in terms of the counterfactual value of leaves and terminals of the tree.
\begin{align}
& = \max_{\pi^*_i} \sum_{t=1}^T & & \left( \sum_{\pubstate{} \in \treefrontier{\mathcal{L}^t}} \sum_{\infostate{i} \in \mathcal{S}_i(\pubstate{})} \left( P_i(\infostate{i} | \pi^*_i) \cfv{\infostate{i}}{\pi^t \leftarrow \pi^*_i} - P_i(\infostate{i} | \pi^t) \cfv{\infostate{i}}{\pi^t} \right) \right. \nonumber \\
& & & \left. + \sum_{\pubstate{} \in \treeterminals{\mathcal{L}^t}} \sum_{z \in I(\pubstate{})} \left( P_i(z | \pi^t \leftarrow \pi^*_i) \cfv{z}{\pi^t} - P_i(z | \pi^t) \cfv{z}{\pi^t} \right) \right) \label{eqn:regret_leaves}
\end{align}
Examining part of the first term inside the sum, we can independently maximise the counterfactual values at each information state $\infostate{i}$. As above, this is equivalent to maximising at public state $\pubstate{}$.
\begin{align*}
& \sum_{\infostate{i} \in \mathcal{S}_i(\pubstate{})} \left( P_i(\infostate{i} | \pi^*_i) \cfv{\infostate{i}}{\pi^t \leftarrow \pi^*_i} - P_i(\infostate{i} | \pi^t) \cfv{\infostate{i}}{\pi^t} \right) \\
& \le \sum_{\infostate{i} \in \mathcal{S}_i(\pubstate{})} \max_{\pi^{**}} \left( P_i(\infostate{i} | \pi^*_i) \cfv{\infostate{i}}{\pi^t \leftarrow \pi^{**}_i} - P_i(\infostate{i} | \pi^t) \cfv{\infostate{i}}{\pi^t} \right) \\
& = \max_{\pi^{**}} \sum_{\infostate{i} \in \mathcal{S}_i(\pubstate{})} \left( P_i(\infostate{i} | \pi^*_i) \cfv{\infostate{i}}{\pi^t \leftarrow \pi^{**}_i} - P_i(\infostate{i} | \pi^t) \cfv{\infostate{i}}{\pi^t} \right)
\end{align*}
Given that we individually maximised over each minuend, we satisfy the requirements of Lemma~\ref{lem:cfv_decompose}. We can then use the value function quality bounds.
\begin{align*}
& \le \max_{\pi^{**}} \sum_{\infostate{i} \in \mathcal{S}_i(\pubstate{})} \left( \cfv{\infostate{i}}{\pi^t \leftarrow \pi^{**}_i} - \cfv{\infostate{i}}{\pi^t} \right) \\
& + \sum_{\infostate{i} \in \mathcal{S}_i(\pubstate{})} \left( P_i(\infostate{i} | \pi^*_i) \cfv{\infostate{i}}{\pi^t} - P_i(\infostate{i} | \pi^t) \cfv{\infostate{i}}{\pi^t} \right) \\
& \le \epsilon + \sum_{\infostate{i} \in \mathcal{S}_i(\pubstate{})} \left( P_i(\infostate{i} | \pi^*_i) \cfv{\infostate{i}}{\pi^t} - P_i(\infostate{i} | \pi^t) \cfv{\infostate{i}}{\pi^t} \right)
\end{align*}
Up to this point, we have used the true counterfactual values for the current policy profile.
At leaves, however, \gtcfr{} only has access to the value function's noisy estimates of the true values.
Applying Lemma~\ref{lem:value_error}, we get
\begin{align*}
& \le \epsilon + \xi + \sum_{\infostate{i} \in \mathcal{S}_i(\pubstate{})} \left( P_i(\infostate{i} | \pi^*_i) \apxcfv{\infostate{i}}{\pi^t} - P_i(\infostate{i} | \pi^t) \apxcfv{\infostate{i}}{\pi^t} \right)
\end{align*}
Placing this back into Equation~\ref{eqn:regret_leaves} and collecting $\epsilon$ and $\xi$ terms, we have
\begin{align*}
\sum_{\infostate{i} \in \mathcal{S}_i(s_0)} R^T_\infostate{i} & \le \sum_{t=1}^T |\treefrontier{\mathcal{L}^t}|(\epsilon + \xi) + \max_{\pi^*_i} \sum_{t=1}^T  \\
& \left( \sum_{\pubstate{} \in \treefrontier{\mathcal{L}^t}} \sum_{\infostate{i} \in \mathcal{S}_i(\pubstate{})} \left( P_i(\infostate{i} | \pi^*_i) \apxcfv{\infostate{i}}{\pi^t} - P_i(\infostate{i} | \pi^t) \apxcfv{\infostate{i}}{\pi^t} \right) \right. \\
& \left. + \sum_{\pubstate{} \in \treeterminals{\mathcal{L}^t}} \sum_{z \in I(\pubstate{})} \left( P_i(z | \pi^t \leftarrow \pi^*_i) \cfv{z}{\pi^t} - P_i(z | \pi^t) \cfv{z}{\pi^t} \right) \right)
\end{align*}
We can rearrange the sums to consider the regret contribution for each public state
\begin{align*}
& = \sum_{t=1}^T |\treefrontier{\mathcal{L}^t}|(\epsilon + \xi) + \max_{\pi^*_i} \sum_{\pubstate{} \in \bigcup_{t=1}^T \mathcal{L}^t} \\
& \left( \sum_{t \text{ s.t. } \pubstate \in \treefrontier{\mathcal{L}^t}} \sum_{\infostate{i} \in \mathcal{S}_i(\pubstate{})} \left( P_i(\infostate{i} | \pi^*_i) \apxcfv{\infostate{i}}{\pi^t} - P_i(\infostate{i} | \pi^t) \apxcfv{\infostate{i}}{\pi^t} \right) \right. \\
& \left. + \sum_{t \text{ s.t. } \pubstate \in \treeterminals{\mathcal{L}^t}} \sum_{z \in I(\pubstate{})} \left( P_i(z | \pi^t \leftarrow \pi^*_i) \cfv{z}{\pi^t} - P_i(z | \pi^t) \cfv{z}{\pi^t} \right) \right)
\end{align*}
As before we can use Lemma~\ref{lem:cfv_decompose} to separate out regrets at the interior states in $\mathcal{N}\coloneqq\treefrontier{\treeinterior{\bigcup_{t=1}^T\mathcal{L}^t}}$, which always depend only on leaves and terminals.
Let $\mathcal{L}^{',t}$ be $\mathcal{L}^t$ minus all public states in $\mathcal{N}$ and any successor states.
\begin{align*}
& \le \sum_{t=1}^T |\treefrontier{\mathcal{L}^t}|(\epsilon + \xi) + \sum_{\pubstate{} \in \mathcal{N}} \sum_{[a,b] \in \interiorintervals{\pubstate}} \sum_{\infostate{i} \in \pubstate{}} \tilde{R}^{a,b}_\infostate{i} + \max_{\pi^*_i} \sum_{\pubstate{} \in \bigcup_{t=1}^T \mathcal{L}^{',t}} \\
& \left( \sum_{t \text{ s.t. } \pubstate \in \treefrontier{\mathcal{L}^{',t}}} \sum_{\infostate{i} \in \mathcal{S}_i(\pubstate{})} \left( P_i(\infostate{i} | \pi^*_i) \apxcfv{\infostate{i}}{\pi^t} - P_i(\infostate{i} | \pi^t) \apxcfv{\infostate{i}}{\pi^t} \right) \right. \\
& \left. + \sum_{t \text{ s.t. } \pubstate \in \treeterminals{\mathcal{L}^{',t}}} \sum_{z \in I(\pubstate{})} \left( P_i(z | \pi^t \leftarrow \pi^*_i) \cfv{z}{\pi^t} - P_i(z | \pi^t) \cfv{z}{\pi^t} \right) \right)
\end{align*}
Note that the states which were separated out are now effectively terminals in smaller trees.
We can repeat this process until regrets for all public states have been separated out.
\begin{align*}
& \le \sum_{t=1}^T |\treefrontier{\mathcal{L}^t}|(\epsilon + \xi) + \sum_{\pubstate{} \in \bigcup_{t=1}^T \mathcal{L}^t} \sum_{[a,b] \in \interiorintervals{\pubstate}} \sum_{\infostate{i} \in \pubstate{}} \tilde{R}^{a,b}_\infostate{i}
\end{align*}
Finally, from bounds on regret-matching\cite{Hart00},
\begin{align*}
& \le \sum_{t=1}^T |\treefrontier{\mathcal{L}^t}|(\epsilon + \xi) + \sum_{\pubstate{} \in \bigcup_{t=1}^T \mathcal{L}^t} |\mathcal{S}_i(\pubstate{})| U \sqrt{A} \sum_{[a,b] \in \interiorintervals{\pubstate}} \sqrt{|[a,b]|}
\end{align*}
\end{proof}

Note that the form of Lemma~\ref{lem:gtcfr_regret} implies that regret might not be sub-linear if public states are repeatedly added and removed from the lookahead tree.
If we only add states and never remove them, however, we get a standard CFR regret bound plus error terms for the value function.
\begin{theorem}
\label{thm:gtcfr_regret_simple}
Assume the conditions of Lemma~\ref{lem:gtcfr_regret} hold, and public states are never removed from the lookahead tree.
Then
\begin{align*}
    R^{T, \text{full}}_i \le \sum_{t=1}^T |\treefrontier{\mathcal{L}^t}|(\epsilon + \xi) + \sum_{\pubstate{} \in \treeinterior{\mathcal{L}^T}} |\mathcal{S}_i(\pubstate{})| U \sqrt{AT}
\end{align*}
\end{theorem}
\begin{proof}
This follows from Lemma~\ref{lem:gtcfr_regret}, noting that the interior of $\mathcal{L}^t$ monotonically grows over time.
\end{proof}

\subsection*{Self-play Values as Re-solving Constraints}
By using a value network in solving, we lose the ability to compute our opponent's counterfactual best response values to our average strategy~\cite{sustr2019montecarlo}.
It is easy to track the opponent's average self-play value across iterations of a CFR variant, but using these values as re-solving constraints does not trivially lead to a bound on exploitability for the re-solved strategy.
We show here that average CFR self-play values lead to reasonable, controllable error bounds in the context of continual re-solving.
We will use $(x)^+$ to mean $\max\{x, 0\}$.
For simplicity, we will also assume that the subgame that is being re-solved is in the \gtcfr{} lookahead tree for all iterations.

\begin{theorem}
Assume we have some average strategy $\bar{\pi}$ generated by $T$ iterations of \gtcfr{} solver using a value function with quality $\epsilon, \xi$, with final lookahead tree $\mathcal{L}^T$ where public states were never removed from the lookahead tree, and a final average regret $R^T_i$ for the player of interest.
Further assume that we have re-solved some public subgame $S$ rooted public state $\pubstate{}$, using the average counterfactual values $\bar{v}(\infostate{o}) \coloneqq \frac{1}{T} \sum_{t=1}^T \cfv{\infostate{o}}{\pi^t}$ as the opt-out values in the re-solving gadget.
Let $\pi^S$ be the strategy generated from the re-solving game, with some player and opponent average regrets $\bar{R}^S_i$ and $\bar{R}^S_o$, respectively.
Then
\begin{align*}
\text{BV}^{\bar{\pi} \leftarrow \pi^S}_o - \text{BV}^{\bar{\pi}}_o \le & (\bar{R}^S_o)^+ + (\bar{R}^S_i)^+ + \bar{R}_i \\
& + 2(\max_t|\treefrontier{\mathcal{L}^t_\pubstate{}}|(\epsilon + \xi) + \sum_{\pubstate{} \in \treeinterior{\mathcal{L}^T_\pubstate{}}} |\mathcal{S}_i(\pubstate{})| U \sqrt{\frac{A}{T}})
\end{align*}
\label{thm:resolving_average_value_constraints}
\end{theorem}
\begin{proof}
The general outline of the proof has two parts, both asking the question "how much can the opponent best response value increase?"
As in Lemma~4 of \cite{Moravcik17DeepStack}, we can consider breaking the error in re-solving opt-out values into separate underestimation and overestimation terms.
The first part of this proof is a bound that takes into account the re-solving solution quality, and how much the average values underestimate the best response to the average.
This underestimation is bounded by the opponent regret at a subgame, which requires the solving algorithm to have low regret everywhere in the game: low regret for the opponent does not directly imply that the opponent has low regret in portions of the game that they do not play.
The second part of the proof is placing a bound on the overestimation, using the player's regret rather than the opponent's regret.

We start by noting that from the opponent player $o$'s point of view, we can replace an information set $\infostate{o}$ with a terminal that has utility $\brcfv{\infostate{o}}{\bar{\pi}}$, and the best response utility $\text{BV}^{\bar{\pi},\infostate{o}\leftarrow\brcfv{\infostate{o}}{\bar{\pi}}}_o$ in this modified game will be equal to $\text{BV}^{\bar{\pi}}_o$.
We can extend this to the entire subgame $S$, replacing each $\infostate{o}$ with a terminal giving the opponent the best response value: $\text{BV}^{\bar{\pi},S\leftarrow\brcfv{S}{\bar{\pi}}}_o = \text{BV}^{\bar{\pi}}_o$.
Using this notation, we can rewrite $\text{BV}^{\bar{\pi} \leftarrow \pi^S}_o$:
\begin{align*}
& \text{BV}^{\bar{\pi} \leftarrow \pi^S}_o - \text{BV}^{\bar{\pi}}_o \\
= & \text{BV}^{\bar{\pi}, S \leftarrow \brcfv{S}{\pi^S}}_o - \text{BV}^{\bar{\pi}}_o 
\end{align*}

Next, note that $\brcfv{\infostate{o}}{\pi^S}$, the opponent's counterfactual best response to the re-solved subgame strategy $\pi^S$ at any $\infostate{o}$ at the root of $S$, is no greater than the value of $\max\{\brcfv{\infostate{o}}{\pi^S},\bar{v}(\infostate{o})\}$, the value of $\infostate{o}$ within the re-solving game before the gadget where the opponent has decision to opt-out for a fixed value $\bar{v}(\infostate{o})$.
That is, adding an extra opponent action which terminates the game never decreases the opponent's best response utility.
Extending this to the entire subgame $S$ again, we get
\begin{align}
& \text{BV}^{\bar{\pi}, S \leftarrow \brcfv{S}{\pi^S}}_o - \text{BV}^{\bar{\pi}}_o \nonumber \\
\le & \text{BV}^{\bar{\pi}, S \leftarrow \max\{\brcfv{S}{\pi^S}, \bar{v}\}}_o - \text{BV}^{\bar{\pi}}_o \label{eq:optout_br_value}
\end{align}

From Lemma 1 of \cite{Moravcik17DeepStack}, the game value of a re-solving game with opt-out values $\bar{v}(\infostate{o})$ is $U^S_{\bar{v}, \bar{\pi}} + \sum_{\infostate{o} \in \mathcal{S}_o(\pubstate{})} \bar{v}(\infostate{o})$, for some underestimation error on the opt-out values that is given by
\begin{align*}
U^S_{\bar{v}, \bar{\pi}} \coloneqq \min_{\pi^{*S}} \sum_{\infostate{o} \in \mathcal{S}_o(\pubstate{})} (\brcfv{\infostate{o}}{\bar{\pi} \leftarrow \pi^{*S}} - \bar{v}(\infostate{o}))^+
\end{align*}
Given the re-solving regrets, we have $\text{BV}^{\pi^S}_o \le (\bar{R}^S_o)^+ + (\bar{R}^S_i)^+ + U^S_{\bar{v}, \bar{\pi}} + \sum_{\infostate{o} \in \mathcal{S}_o(\pubstate{})} \bar{v}(\infostate{o})$.
Because $\text{BV}^{\bar{\pi},\infostate{o}\leftarrow w+\epsilon}_o \le \text{BV}^{\bar{\pi},\infostate{o}\leftarrow w}_o + \epsilon$ for $\epsilon \ge 0$, we can use this inequality to update Equation~\ref{eq:optout_br_value}. That is, there is some per-information-set values $\mathbf{\epsilon}$ such that $\brcfv{\tilde{S}}{\pi^S} = \bar{v}(\cdot) + \mathbf{\epsilon}$ and $\mathbf{\epsilon}\cdot\mathbf{1} \le (\bar{R}^S_o)^+ + (\bar{R}^S_i)^+ + U^S_{\bar{v}, \bar{\pi}}$, so that

\begin{align}
& \text{BV}^{\bar{\pi}, S \leftarrow \max\{\brcfv{S}{\pi^S}, \bar{v}\}}_o - \text{BV}^{\bar{\pi}}_o \nonumber \\
= & \text{BV}^{\bar{\pi}, S \leftarrow \bar{v} + \mathbf{\epsilon}}_o - \text{BV}^{\bar{\pi}}_o \nonumber \\
\le & \text{BV}^{\bar{\pi}, S \leftarrow \bar{v}}_o + \mathbf{\epsilon} \cdot \mathbf{1} - \text{BV}^{\bar{\pi}}_o \nonumber \\
\le & \text{BV}^{\bar{\pi}, S \leftarrow \bar{v}}_o + (\bar{R}^S_o)^+ + (\bar{R}^S_i)^+ + U^S_{\bar{v}, \bar{\pi}} - \text{BV}^{\bar{\pi}}_o \label{eq:br_value_with_underestimation}
\end{align}

Looking at $U^S_{\bar{v}, \bar{\pi}}$, we note that this minimum is no greater than the case when $\pi^*=\bar{\pi}$.
The difference $\brcfv{\infostate{o}}{\bar{\pi} \leftarrow \pi^{*S}} - \bar{v}(\infostate{o})$ is the average full counterfactual regret $R_\infostate{o}$ of strategy $\bar{\pi}$ at $\infostate{o}$.
Restricting our attention to $\mathcal{L}^t_\pubstate{}$, the portion of the lookahead tree restricted to $\pubstate{}$ and its descendants, Theorem~\ref{thm:gtcfr_regret_simple} gives us a bound on $U^S_{\bar{v}, \bar{\pi}}$ and we can update Equation~\ref{eq:br_value_with_underestimation}
\begin{align}
& \text{BV}^{\bar{\pi}, S \leftarrow \bar{v}}_o + (\bar{R}^S_o)^+ + (\bar{R}^S_i)^+ + U^S_{\bar{v}, \bar{\pi}} - \text{BV}^{\bar{\pi}}_o \nonumber \\
\le & \text{BV}^{\bar{\pi}, S \leftarrow \bar{v}}_o  - \text{BV}^{\bar{\pi}}_o \label{eq:br_value_subgame_eliminated} \\
& + (\bar{R}^S_o)^+ + (\bar{R}^S_i)^+ \max_t|\treefrontier{\mathcal{L}^t_\pubstate{}}|(\epsilon + \xi) + \sum_{\pubstate{} \in \treeinterior{\mathcal{L}^T_\pubstate{}}} |\mathcal{S}_i(\pubstate{})| U \sqrt{\frac{A}{T}} \nonumber
\end{align}

Looking at just the difference in opponent counterfactual best response values, we can again get an upper bound by giving the opponent the choice at all information sets at the root of subgame $S$ of playing a best response against the unmodified strategy $\bar{\pi}$ to get value $\brcfv{S}{\bar{\pi}}$, or opting out to get value $\bar{v}$.
\vspace{-1em}
\begin{align}
& \text{BV}^{\bar{\pi}, S \leftarrow \bar{v}}_o  - \text{BV}^{\bar{\pi}}_o \nonumber \\
\le & \text{BV}^{\bar{\pi}, S \leftarrow \max\{\brcfv{S}{\bar{\pi}}, \bar{v}\}}_o - \text{BV}^{\bar{\pi}}_o \nonumber \\
= & (\text{BV}^{\bar{\pi}, S \leftarrow \max\{\brcfv{S}{\bar{\pi}}, \bar{v}\}}_o - \bar{v}_o) - (\text{BV}^{\bar{\pi}}_o - \bar{v}_o) \nonumber \\
\le & (\text{BV}^{\bar{\pi}, S \leftarrow \max\{\brcfv{S}{\bar{\pi}}, \bar{v}\}}_o - \bar{v}_o) - (\text{BV}^{\pi^*}_o - \bar{v}_o) \nonumber \\
= & (\text{BV}^{\bar{\pi}, S \leftarrow \max\{\brcfv{S}{\bar{\pi}}, \bar{v}\}}_o - \bar{v}_o) - (v^{\pi^*}_o - \bar{v}_o) \nonumber \\
= & (\text{BV}^{\bar{\pi}, S \leftarrow \max\{\brcfv{S}{\bar{\pi}}, \bar{v}\}}_o - \bar{v}_o) + (v^{\pi^*}_i - \bar{v}_i) \nonumber \\
\le & (\text{BV}^{\bar{\pi}, S \leftarrow \max\{\brcfv{S}{\bar{\pi}}, \bar{v}\}}_o - \bar{v}_o) + (\text{BV}^{\bar{\pi}}_i - \bar{v}_i) \nonumber \\
= & (\text{BV}^{\bar{\pi}, S \leftarrow \max\{\brcfv{S}{\bar{\pi}}, \bar{v}\}}_o - \bar{v}_o) + \bar{R}_i \label{eq:br_value_final_optout}
\end{align}

The difference of the first two terms is the regret in the opt-out game game described above, where we have lifted each iteration strategy $\pi^t$ into this game by never selecting the opt-out choice.
Consider the immediate counterfactual regret $\tilde{R}^T(\infostate{o})$ in this situation for any information state $\infostate{o}$ in this augmented game.
Writing this in terms of the original immediate counterfactual regret $R^T(\infostate{o})$ and the opt-out value, we get
\begin{align*}
\tilde{R}^T(\infostate{o}) & = \max\{T(\bar{v}[\infostate{o}] - \bar{v}[\infostate{o}]), R^T(\infostate{o})\} \\
& = (R^T(\infostate{o}))^+
\end{align*}
Because the positive immediate regret in the opt-out game is the same as the positive regret in the original game, we can use the Theorem~\ref{thm:gtcfr_regret_simple} bound, which is composed from immediate regrets.
Putting this together with Equation~\ref{eq:br_value_subgame_eliminated} and Equation~\ref{eq:br_value_final_optout}, we get
\begin{align*}
& \text{BV}^{\bar{\pi} \leftarrow \pi^S}_o - \text{BV}^{\bar{\pi}}_o \\
\le & (\bar{R}^S_o)^+ + (\bar{R}^S_i)^+ + \bar{R}_i \\
& + 2(\max_t|\treefrontier{\mathcal{L}^t_\pubstate{}}|(\epsilon + \xi) + \sum_{\pubstate{} \in \treeinterior{\mathcal{L}^T_\pubstate{}}} |\mathcal{S}_i(\pubstate{})| U \sqrt{\frac{A}{T}})
\end{align*}
\end{proof}

\subsection*{Continual Re-solving}
Continual re-solving puts \gtcfr{} together with re-solving the previously solved subgame.
A bound on final solution quality follows directly from applications of Theorem~\ref{thm:gtcfr_regret_simple} and Theorem~\ref{thm:resolving_average_value_constraints}.

\begin{theorem}
Assume we have played a game using continual re-solving, with one initial solve and $D$ re-solving steps.
Each solving or re-solving step finds an approximate Nash equilibrium through $T$ iterations of \gtcfr{} using a value function with quality $\epsilon, \xi$, public states are never removed from the lookahead tree, the maximum interior size $\sum_{\pubstate{} \in \treeinterior{\mathcal{L}^T}} |\mathcal{S}_i(\pubstate)|$ of all lookahead trees is bounded by $N$, the sum of frontier sizes across all lookahead trees is bounded by $F$, the maximum number of actions at any information sets is $A$, and the maximum difference in values between any two strategies is $U$.
The exploitability of the final strategy is then bounded by
$(5D+2)\left(F(\epsilon + \xi) + NU \sqrt{\frac{A}{T}}\right)$.
\end{theorem}
\begin{proof}
The exploitability $\text{EXP}_0$ of the player's initial strategy from the original solve is bounded by the sum of the regrets for both players.
Theorem~\ref{thm:gtcfr_regret_simple} provides regret bounds for \gtcfr{}, so
\begin{align*}
\text{EXP}_0 & \le 2\left(F(\epsilon + \xi) + NU \sqrt{\frac{A}{T}}\right)
\end{align*}

Each subsequent re-solve is operating on the strategy of the previous step, using the average values for the opt-out values.
That is, the first re-solve will be updating the strategy from the initial solve, the second re-solve will be updating the subgame strategy from the first re-solve, and so on.
Theorem~\ref{thm:resolving_average_value_constraints} provides a bound on how much the exploitability increases after each re-solving step, with Theorem~\ref{thm:gtcfr_regret_simple} providing the necessary regret bounds
\begin{align*}
\text{EXP}_d \le & \text{EXP}_{d-1} + (\bar{R}^S_o)^+ + (\bar{R}^S_i)^+ + \bar{R}_i + 2\left(F(\epsilon + \xi) + NU \sqrt{\frac{A}{T}}\right) \\
\le & \text{EXP}_{d-1} + 5\left(F(\epsilon + \xi) + NU \sqrt{\frac{A}{T}}\right)
\end{align*}

Unrolling for $D$ re-solving steps leads to the final bound.
\end{proof}

\clearpage
\begin{figure}[t!]
  \centering
  \includegraphics[width=0.7\textwidth]{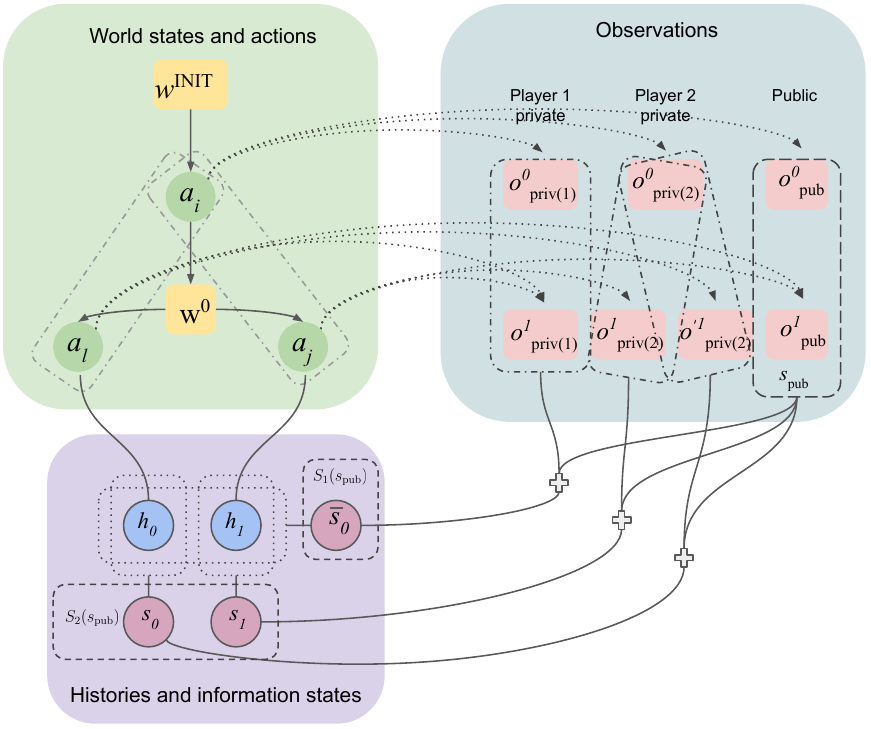}
  \caption{\textbf{An example of a Factored-Observation Stochastic Game (FOSG).} This figure presents the visual view of notation from \sectionref{sec:background}. In this example the game starts in  $w^{init}$ which is the complete state of the environment containing private information for both players. After playing action $a_i$ the state moves to $w^0$ where there are two possible actions. Each action emits private and public observations. In this example, actions $a_j$ and $a_l$ emit the same private observation $o^1_{\text{priv}(1)}$ for player 1, therefore they cannot distinguish which action happened. On the other hand, player 2 has different observations $o^1_{\text{priv}(2)}$ and $o'^1_{\text{priv}(2)}$ for each of the actions, therefore they have more information about the state of the environment than player 1. The sequence  of public observations shared by both players information is denoted as $s_\text{pub}$. Both sequences of actions and factored observations meet in the final `Histories and information states' view. The two possible action sequences are represented by histories $h_0$ and $h_1$, where $h_0 = (a_i, a_l)$, $h_1 = (a_i, a_j)$. Since both actions $a_l$ and $a_j$ result in the same observation for player 1, they cannot tell which one of the histories happened and his information state $\bar{s}_0$ contains them both. This is not the case for player 2, who can separate the histories, and each of his information states $s_0$ and $s_1$ contains just one history.}
  \label{fig:fosg-example}
\end{figure}

\clearpage
\begin{figure}[b!]
  \centering
  \includegraphics[width=0.4\textwidth]{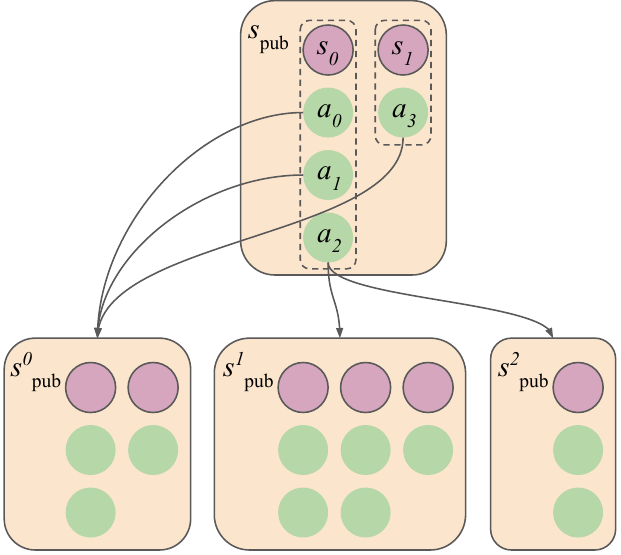}
  \caption{\textbf{An example of a public tree.} The public tree provides different view of the FOSG. In this example actions $a_0$ and $a_1$ emit the same public observation and therefore they lead to the same public tree node  $s^0_\text{pub}$. On the other hand, action $a_2$ can lead to multiple possible states: for instance when a detective in Scotland Yard moves to a location the game can either 1) end because Mr.~X was there and he was caught or 2) it continues because he was in a different station.}
  \label{fig:fosg-public-tree}
\end{figure}

\clearpage
\begin{figure}[t]
  \centering
  \includegraphics[angle=0,width=0.65\textwidth]{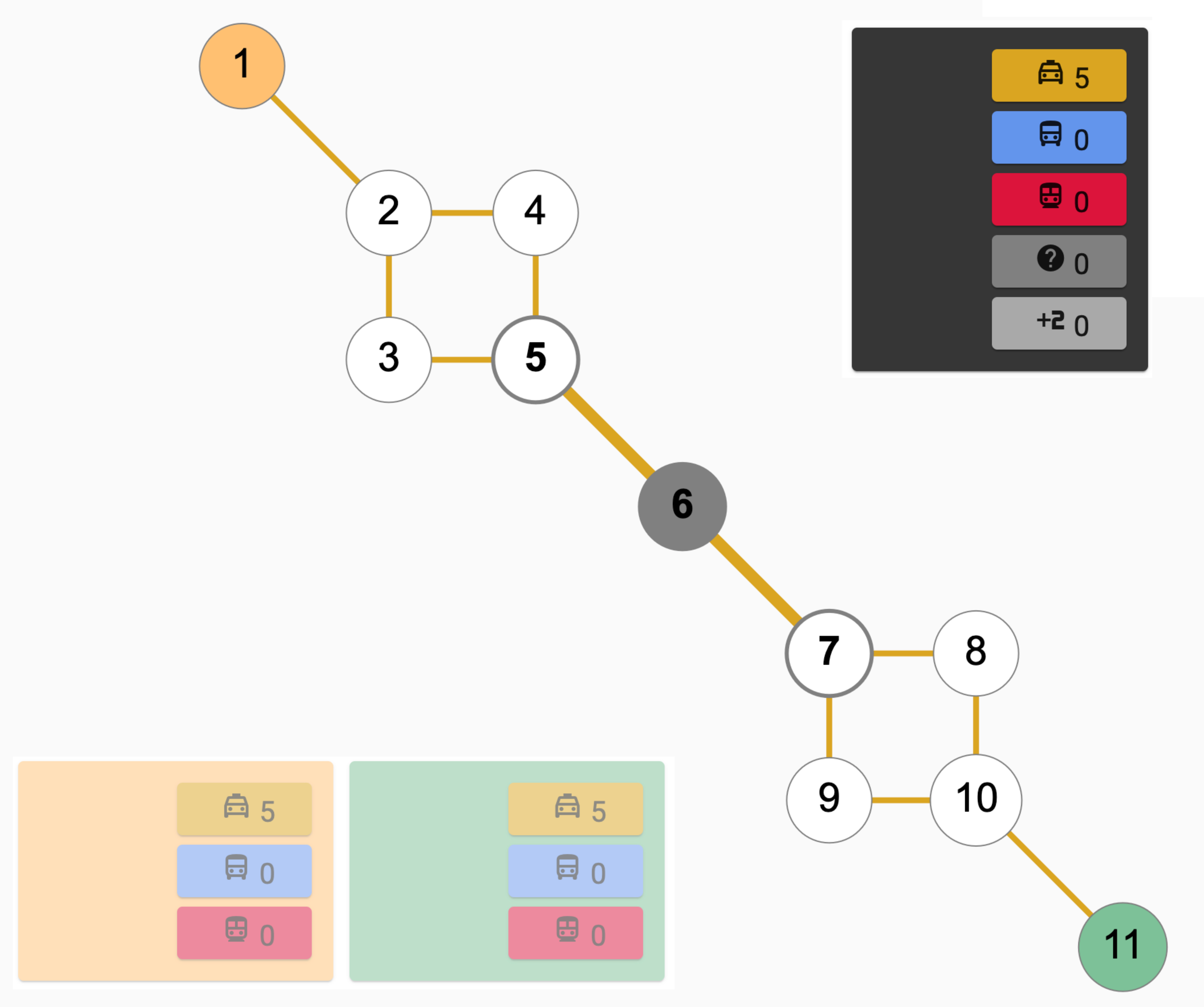}
  \caption{\textbf{Initial situation on the glasses map for Scotland Yard.} Mr.~X starts at station 6 while the two detectives start at stations 1 and 11. All of them have 5 taxi cards (all edges in this map are of the same type) and the game is played for 5 rounds.}
  \label{fig:scotlandyard-glasses}
\end{figure}

\clearpage
\begin{figure}
\centering
\includegraphics[width=1.0\textwidth]{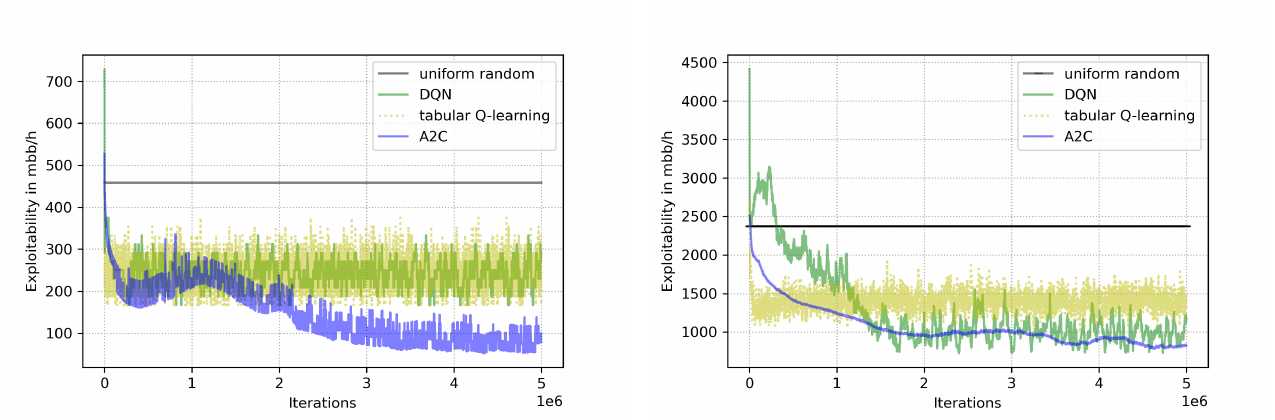}
\begin{subfigure}{.5\textwidth}
  \centering
  \caption{Exploitability in 2-player Kuhn poker}
  \label{fig:sub1}
\end{subfigure}%
\begin{subfigure}{.5\textwidth}
  \centering
  \caption{Exploitability in 2-player Leduc poker}
  \label{fig:sub2}
\end{subfigure}
\caption{\textbf{Comparing performance of DQN, A2C, tabular Q-learning and uniform random policy in (A) Kuhn poker and (B) Leduc poker.}}
\label{fig:rl-exploitability}
\end{figure}

\clearpage
\begin{table*}[ht]
\centering
\renewcommand{\arraystretch}{1.3}
\begin{tabularx}{\textwidth}{l|XXX}
\toprule
 Game	& Architecture	& Belief features	& Public state features\\  \midrule
Chess	& ResNet	& Redundant --- there is no uncertainty over players state. 	&  One 8x8 plane for each piece type (6) of each player (2) and repetitions planes (2) for last eight moves + scalar planes (7), 119 8x8 planes in total. \\
Go	& ResNet	& Redundant 	&  One 19x19 plane for stones of each player (2) for last eight moves plus a single plane encoding player to act, 17 19x19 planes in total. \\
Poker	& MLP 6 x 2048  	& 1326 (possible private card combinations) * 2 (num of players).	& N hot encoding of board cards (52) + commitment of each player normalized by his stack (2) + 1 hot encoding of who acts next, including chance player (3). \\
Scotland Yard	& MLP 6 x 512 	& 1 (detectives' position is always certain) + 199 (possible Mr X's position).	& 1 hot encoding of position of each detective (5*199) + cards of each detective (5*3) + cards of Mr X (5) + who is playing next (6) + was double move just used (1) + how many rounds were played (1). \\
\bottomrule
\end{tabularx}
\caption{\textbf{A neural network architecture and features used for each game.}}
\label{tab:nn-architectures-per-game}
\end{table*}

\clearpage
\begin{table*}[h!]
\centering
\renewcommand{\arraystretch}{1.3}
\begin{tabular}{ll|llll}
\toprule
                        Hyperparam &                   Symbol &  Chess &    Go & Scot. Yard &    HUNL \\
\midrule
                        Batch size &                          &   2048 &  2048 &       1024 &    1024 \\
                         Optimizer &                          &    sgd &   sgd &        sgd &    adam \\
        Initial learning rate (LR) &          $\alpha_{init}$ &    0.1 &  0.02 &        0.1 &  0.0001 \\
                    LR decay steps &              $T_{decay}$ &    40k &  200k &         2M &      2M \\
                     LR decay rate &                      $d$ &    0.8 &   0.1 &        0.5 &     0.5 \\
                Policy head weight &                    $w_p$ &      1 &     1 &       0.05 &    0.01 \\
                 Value head weight &                    $w_v$ &   0.25 &   0.5 &          1 &       1 \\
                Replay buffer size &                          &    50M &   50M &         1M &      1M \\
      Max grad updates per example &                          &      1 &   0.2 &          5 &      10 \\
     TD(1) target sample probability &                $p_{td1}$ &      0 &   0.2 &          0 &       0 \\
                Queries per search &             $q_{search}$ &      1 &     0 &        0.3 &     0.9 \\
      Recursive queries per search &          $q_{recursive}$ &    0.2 &     0 &        0.1 &     0.1 \\
       Self-play uniform policy mix &               $\epsilon$ &      0 &     0 &          0 &     0.1 \\
                    Resign enabled &                          &   True &  True &      False &   False \\
                  Resign threshold &      $resign\_threshold$ &   -0.9 &  -0.9 &          - &       - \\
 Min ratio of games without resign &         $p_{no\_resign}$ &    0.2 &   0.2 &          - &       - \\
            Greedy play after move &  $moves_{greedy\_after}$ &     30 &    30 &      never &   never \\
          Max moves in one episode &            $moves_{max}$ &    512 &   722 &     unlim. &  unlim. \\
         Prior softmax temperature &              $T_{prior}$ &    1.5 &   1.5 &          1 &       1 \\
\bottomrule
\end{tabular}
\caption{\textbf{Hyperparameters for each game.}}
\label{tab:hparams-per-game}
\end{table*}

\clearpage
\begin{table}[t!]
\centering
\begin{tabular}{lr}
\toprule
                              Agent &  Rel. Elo \\
\midrule
AlphaZero(s=16k, t=800k) & +3139 \\
AlphaZero(s=16k, t=400k) & +3021 \\
AlphaZero(s=8k, t=800k) & +2875 \\
AlphaZero(s=8k, t=400k) & +2801 \\
AlphaZero(s=4k, t=800k) & +2643 \\
AlphaZero(s=16k, t=200k) & +2610 \\
AlphaZero(s=4k, t=400k) & +2584 \\
AlphaZero(s=2k, t=800k) & +2451 \\
AlphaZero(s=8k, t=200k) & +2428 \\
AlphaZero(s=2k, t=400k) & +2353 \\
AlphaZero(s=4k, t=200k) & +2234 \\
AlphaZero(s=800, t=800k) & +2099 \\
AlphaZero(s=16k, t=100k) & +2088 \\
AlphaZero(s=2k, t=200k) & +2063 \\
AlphaZero(s=800, t=400k) & +2036 \\
\textbf{SoG(s=16k, c=10)} & \textbf{+1970} \\
AlphaZero(s=8k, t=100k) & +1940 \\
\textbf{SoG(s=8k, c=10)} & \textbf{+1902} \\
AlphaZero(s=800, t=200k) & +1812 \\
\textbf{SoG(s=4k, c=10)} & \textbf{+1796} \\
AlphaZero(s=4k, t=100k) & +1783 \\
\textbf{SoG(s=2k, c=10)} & \textbf{+1672} \\
AlphaZero(s=2k, t=100k) & +1618 \\
\textbf{SoG(s=800, c=1)} & \textbf{+1426} \\
AlphaZero(s=800, t=100k) & +1360 \\
Pachi(s=100k) & +869 \\
Pachi(s=10k) & +231 \\
GnuGo(l=10) & +0 \\
\bottomrule
\end{tabular}
\caption{\textbf{Full Go results (Non Recursive Queries).} Elo of GnuGo with a single thread and 100ms thinking time was set to be 0. AlphaZero(s=16k, t=800k) refers to 16000 search simulations after 800000 training steps.}
\label{tab:full-go-results}
\end{table}

\clearpage
\begin{table}[t!]
\centering
\begin{tabular}{lr}
\toprule
                              Agent &  Rel. Elo \\
\midrule
AlphaZero(s=16k, t=800k) & +3431 \\
AlphaZero(s=16k, t=400k) & +3319 \\
AlphaZero(s=8k, t=800k) & +3169 \\
AlphaZero(s=8k, t=400k) & +3093 \\
AlphaZero(s=4k, t=800k) & +2933 \\
AlphaZero(s=16k, t=200k) & +2899 \\
AlphaZero(s=4k, t=400k) & +2880 \\
AlphaZero(s=2k, t=800k) & +2745 \\
AlphaZero(s=8k, t=200k) & +2712 \\
AlphaZero(s=2k, t=400k) & +2643 \\
AlphaZero(s=4k, t=200k) & +2509 \\
AlphaZero(s=800, t=800k) & +2394 \\
AlphaZero(s=16k, t=100k) & +2391 \\
AlphaZero(s=2k, t=200k) & +2348 \\
AlphaZero(s=800, t=400k) & +2315 \\
AlphaZero(s=8k, t=100k) & +2240 \\
AlphaZero(s=800, t=200k) & +2105 \\
AlphaZero(s=4k, t=100k) & +2078 \\
\textbf{SoG(s=16k, c=10)} & \textbf{+2025} \\
\textbf{SoG(s=8k, c=10)} & \textbf{+1937} \\
AlphaZero(s=2k, t=100k) & +1928 \\
\textbf{SoG(s=4k, c=10)} & \textbf{+1838} \\
\textbf{SoG(s=2k, c=10)} & \textbf{+1766} \\
AlphaZero(s=800, t=100k) & +1644 \\
\textbf{SoG(s=800, c=1)} & \textbf{+1579} \\
Pachi(s=100k) & +958 \\
Pachi(s=10k) & +227 \\
GnuGo(l=10) & +0 \\
\bottomrule
\end{tabular}
\caption{\textbf{Full Go results (Recursive Queries).} Elo of GnuGo with a single thread and 100ms thinking time was set to be 0. AlphaZero(s=16k, t=800k) refers to 16000 search simulations after 800000 training steps.}
\label{tab:full-go-results-recursive}
\end{table}

\clearpage
\begin{table*}
\begin{center}
\begin{tabular}{l|l||lll}
{\bf Num. Sims} & {\bf UCT const. $(C)$} & {\bf Expl. (mvd)} & {\bf Expl. (mvis)} & {\bf Expl. (mval)} \\
\hline
   10 &        1.0 & 2168 & 2449 & 2173\\
   10 &        2.0 & 2058 & 2408 & 2341\\
   10 &        5.0 & 1902 & 2615 & 2517\\
   10 &       10.0 & 1738 & 2555 & 2360\\
   10 &       13.0 & 1799 & 2517 & 2598\\
   10 &       20.0 & 1821 & 2830 & 2349\\
   10 &       26.0 & 1888 & 2861 & 2669\\
\hline
  100 &        1.0 & 1489 & 1509 & 1333\\
  100 &        2.0 & 1404 & 1587 & 1395\\
  100 &        5.0 & 1239 & 1145 & 1094\\
  100 &       10.0 & 1213 & 1195 & 1245\\
  100 &       13.0 & 1218 & 1292 & 1227\\
  100 &       20.0 & 1350 & 1456 & 1342\\
  100 &       26.0 & 1448 & 1747 & 1568\\
\hline
 1000 &        1.0 & 1323 & 1218 & 1177\\
 1000 &        2.0 & 1069 & 1212 & 864\\
 1000 &        5.0 & 699 & 778 & 681\\
 1000 &       10.0 & 697 & 601 & 632\\
 1000 &       13.0 & 741 & 759 & 744\\
 1000 &       20.0 & 859 & 962 & 991\\
 1000 &       26.0 & 966 & 1029 & 1057\\
\hline
10000 &        1.0 & 1348 & 948 & 1134\\
10000 &        2.0 & 911 & 877 & 763\\
10000 &        5.0 & 516 & 582 & 538\\
10000 &       10.0 & 490 & 485 & 480\\
10000 &       13.0 & 511 & 465 & 470\\
10000 &       20.0 & 572 & 505 & 505\\
10000 &       26.0 & 631 & 575 & 570\\
\hline
\end{tabular}
\caption{\textbf{Average exploitability (in mbb/h) over five policy constructions obtained by independent searches of IS-MCTS runs at each
information state in Leduc Poker.}
The parameter $C$ is the value of the UCT exploration constant.
The final policy is obtained either by normalizing the visit counts (mvd),
choosing the action with maximum visits (mvis), or choosing the action with the maximal Monte Carlo value estimate
(mval).}
\label{fig:leduc-ismcts-exploitability}
\end{center}
\end{table*}

\clearpage
\begin{table*}
\begin{center}
\begin{tabular}{l|l|l|l}
{\bf Parameter} & {\bf DQN} & {\bf Tabular} & {\bf A2C} \\
{} & {} & {\bf Q-Learning} & {} \\
\hline
   \small{Learning rate (lr)} & \footnotesize{1e-1, 1e-2, 1e-3, 1e-4}  & \footnotesize{NA} & \small{Actor lr:}  \footnotesize{1e-3, 1e-4, 1e-5}\\ 
    & &  & \small{Critic lr:} \footnotesize{1e-2, 1e-3}\\ \hdashline
   \small{Decaying exploration rate} & \footnotesize{1., 0.8, 0.5, 0.2, 0.1} &  \footnotesize{NA} &  \footnotesize{NA} \\ \hdashline
   \small{Replay buffer size} & \footnotesize{100, 1000, 10000, 100000} & \footnotesize{NA} & \footnotesize{NA} \\ \hdashline
   \small{Hidden layer size} & \footnotesize{'32', '64', '128',}  & \footnotesize{NA} &  \footnotesize{'32', '64', '128',}\\
   & \footnotesize{'32, 32', '64, 64'}  & \footnotesize{NA} &  \footnotesize{'32, 32', '64, 64'}\\ \hdashline
   \small{Num. of critic updates} & \footnotesize{NA} &  \footnotesize{NA} &  \footnotesize{4, 8, 16}\\
   \small{before every actor update} &  &  &  \\ \hdashline
   \small{Step size} & \footnotesize{NA} & \footnotesize{0.1, 0.2, 0.5,} &  \\
   & & \footnotesize{0.8, 1.0} &  \\
\hline
\end{tabular}
\caption{\textbf{Hyper parameters swept over in each RL algorithm.}}
\label{fig:RL_hyper_params}
\end{center}
\end{table*}

\fi 

\end{document}